\documentclass{article}

\usepackage{microtype}
\usepackage{graphicx}
\usepackage{booktabs} 
\usepackage{xcolor}

\usepackage{hyperref}

\usepackage[accepted]{icml2021}

\icmltitlerunning{Delayed Reinforcement Learning by Imitation}

\usepackage{mymacros}

\begin{document}

\twocolumn[
\icmltitle{Delayed Reinforcement Learning by Imitation}



\icmlsetsymbol{equal}{*}

\begin{icmlauthorlist}
\icmlauthor{Pierre Liotet}{equal,poli}
\icmlauthor{Davide Maran}{equal,poli}
\icmlauthor{Lorenzo Bisi}{poli}
\icmlauthor{Marcello Restelli}{poli}
\end{icmlauthorlist}

\icmlaffiliation{poli}{Politecnico di Milano, Milan, Italy}

\icmlcorrespondingauthor{Pierre Liotet}{pierre.liotet@polimi.it}


\vskip 0.3in
]



\printAffiliationsAndNotice{\icmlEqualContribution} 

\begin{abstract}
When the agent's observations or interactions are delayed, classic reinforcement learning tools usually fail. 
In this paper, we propose a simple yet new and efficient solution to this problem.
We assume that, in the undelayed environment, an efficient policy is known or can be easily learned, but the task may suffer from delays in practice and we thus want to take them into account.
We present a novel algorithm, Delayed Imitation with Dataset Aggregation (DIDA), which builds upon imitation learning methods to learn how to act in a delayed environment from undelayed demonstrations.
We provide a theoretical analysis of the approach that will guide the practical design of DIDA. 
These results are also of general interest in the delayed reinforcement learning literature by providing bounds on the performance between delayed and undelayed tasks, under smoothness conditions.
We show empirically that DIDA obtains high performances with a remarkable sample efficiency on a variety of tasks, including robotic locomotion, classic control, and trading.
\end{abstract}

\section{Introduction}
In reinforcement learning (RL), it is generally assumed that the effect of an action over the environment is known instantaneously to the agent. 
However, in the presence of delays, this classic setting is challenged. The effect of a delayed action execution or state observation, if not accounted for, can have perilous effects in practice \citep{dulac2019challenges}. It can induce a performance loss in trading \citep{wilcox1993effect}, create instability in dynamic systems \citep{dugard1998stability, gu2003survey}, be detrimental to the training of real-world robots
\citep{mahmood2018setting}.
To further grasp the importance of delay, one may notice that most traffic laws around the world base safety distances on drivers' ``reaction time'', which is partly due to the perception of the event and partly to the implementation of the action \cite{drozdziel2020drivers}.
These two types of delay have an exact correspondence in RL, where they are dubbed as state observation and action execution delays. 
There are many ways in which delays can further vary.
They may be anonymous (i.e., not known to the agent), constant or stochastic, integer or non-integer. 
In this work, as in most of the literature, we focus on constant non-anonymous delays in the action execution or, equivalently \cite{katsikopoulos2003markov}, in the state observation.

Previous research can be divided into three main directions. In \emph{memoryless} approaches the agent's policy depends on the last observed state \citep{schuitema2010control}. \emph{Augmented} approaches try to cast the problem into a Markov decision process (MDP) by building policies based on an augmented state, composed of the last observed state and on the actions that the agent knows it has taken since then \citep{bouteiller2020reinforcement}. 
A last line of research, which we refer to as \emph{model-based} approach, considers using as a policy input any statistics about the current unknown state that can be computed from the augmented state~\citep{walsh2009learning,firoiu2018human,chen2021delay,agarwal2021blind,derman2021acting,liotet2021learning}. 
The goal of this approach is to avoid the curse of dimensionality posed by the augmented state~\citep{walsh2009learning,bouteiller2020reinforcement}. 
We formalize the problem of delays in \Cref{sec:delays} and give a more in-depth description of the literature in \Cref{sec:related}.


We adopt the simple yet practically effective idea of learning a policy in a delayed environment by applying \emph{imitation learning} to a policy learned in the undelayed environment, as described in Section~\ref{sec:imitation}. 
Not any imitation learning algorithm would work to this aim and DA\textsc{gger}~\citep{ross2011reduction} is particularly well suited for being one of the few algorithms to compute the loss under the learner's own distribution~\cite{osa2018algorithmic}, which is important due to the shift in distribution induced by the delay.
We provide a theoretical analysis (Section~\ref{sec:theoretical}) which, under \emph{smoothness} conditions over the MDP, bounds the performance lost by introducing delays.
Finally, we provide an extensive experimental analysis (Section~\ref{sec:experiments}) where our algorithm is compared with state-of-the-art approaches on a variety of delayed problems and demonstrates great performances and sample efficiency.


\section{Preliminaries}
\label{sec:prelim}

\textbf{Reinforcement Learning}
~~A discrete-time discounted Markov Decision Process (MDP)~\citep{puterman2014markov} is a 6-tuple $\mathcal{M} = (\Ss,\As, p, R, \gamma, \mu)$ where $\Ss $ and $\As$ are measurable sets of states and actions respectively, $p(s'\vert s,a)$ is the probability to transition to a state $s'$ departing from state $s$ and taking action $a$, $R(s,a)$ is a random variable defining the reward collected during such a transition. We denote by $r(s,a)$ its expected value.
Finally, $\mu$ is the initial state distribution.  
The agent's goal is to find a policy $\pi$, which assigns probabilities to the actions given a state, to maximize the \emph{expected discounted return} with discount factor $\gamma\in [0,1)$, defined as\footnote{In the sequel, we will tacit that $s_{t+1}\sim p(\cdot\vert s_t,a_t)$.}
\begin{align}
\label{eq:rl_obj}
    J(\pi) = \E_{\substack{s_{t+1}\sim p(\cdot\vert s_t,a_t)\\a_t\sim\pi(\cdot\vert s_t)\\s_0 \sim \mu(\cdot)}}\left[\sum_{t=0}^H \gamma^t R(s_t,a_t)\right].
\end{align}
We consider an infinite horizon setting, where $H=\infty$. Note that $\frac{1}{1-\gamma}$ can be seen as the effective horizon in this case.
We restrict the set of policies to the \textit{stationary Markovian} policies, $\Pi$, as it contains the optimal one \citep{puterman2014markov}.
RL analysis frequently introduces the concept of state-action value function, which quantifies the expected return obtained under some policy, starting from a given state and fixing the first action. 
Formally, this function is defined as 
\begin{align}
    Q^\pi(s,a) = \E_{\substack{a_t\sim\pi(\cdot\vert s_t)}}
    \left[\sum_{t=0}^H \gamma^t R(s_t,a_t)\bigg | \substack{s_0=s, \\a_0=a}\right].\label{eq:action_state_val}
\end{align}
Similarly we define the state value function as $V^\pi(s)=\E_{a\sim\pi(\cdot\vert s)}[Q^\pi(s,a)]$.
Lastly, we consider the discounted visited state distribution under some policy $\pi$, starting from any initial distribution $\rho$, for some state $s\in\Ss$ as 
\begin{align*}
    d_{\rho}^\pi(s) = (1-\gamma)\sum_{t=0}^\infty \gamma^t\Proba(s_t=s\vert \pi, \rho).
\end{align*}

\textbf{Lipschitz MDPs}~~We now introduce notions that will allow us to characterize the smoothness of an MDP.
Let $L>0$ and let $(X,\dist_X)$ and $(Y,\dist_Y)$ be two metric spaces. A function $f:X \rightarrow Y$ is said to be $L$-Lipschitz continuous ($L$-LC) if, $\forall x,x'\in X$, $\dist_Y(f(x),f(x'))\leq L \dist_X(x,x')$.
We denote the Lipschitz semi-norm of a function $f$ as $\left\Vert f\right\Vert_L = \sup_{x,x'\in X, x\neq x'} \frac{\dist_Y(f(x),f(x'))}{\dist_X(x,x')}$. In real space $X\subset \mathbb{R}^n$, we use as distance the Euclidean one, i.e., $\dist_{X}(x,x')=\lVert x-x'\rVert_2$.
As for the probabilities, we use the $L_1$-Wasserstein distance, which for some probabilities $\mu,\nu$ with sample space $\Omega$ is \citep{villani2009optimal}:
\begin{align*}
    \wass(\mu , \nu)=\sup_{\left\Vert f\right\Vert_L\leq 1} \left\vert \int_{\Omega} f(\omega)(\mu-\nu) (d\omega) \right\vert.
\end{align*}
We now use those concepts to quantify the smoothness of an MDP.
\begin{defi}[Lipschitz MDP]
An MDP is said to be $(L_P,L_r)$-LC if, for all $(s,a),(s',a')\in\Ss\times\As$
\begin{align*}
    & \wass(p(\cdot\vert s,a), p(\cdot\vert s',a'))\leq L_P \left( \dist_{\Ss}(s,s')  + \dist_{\As}(a,a')\right),
    \\
    & \left\vert r(s,a)-r(s',a') \right\vert \leq L_r \left( \dist_{\Ss}(s,s') + \dist_{\As}(a,a')\right).
\end{align*}
\end{defi}
\begin{defi}[Lipschitz policy]
A stationary Markovian policy $\pi$ is said to be $L_\pi$-LC if, $\forall s,s'\in\Ss$ 
\begin{align*}
    \wass(\pi(\cdot\vert s), \pi(\cdot\vert s'))\leq L_\pi \dist_{\Ss}(s,s').
\end{align*}
\end{defi}

These concepts provide useful tools for theoretical analysis and have been extensively used in the field of RL \citep{rachelson2010locality}. Under the assumption of $(L_P,L_r)$-LC MDP and $L_\pi$-LC policy $\pi$, provided that $\gamma L_P(1+L_\pi)\le 1$, then $Q^\pi$ is $L_Q$-LC with $L_Q=\frac{L_r}{1-\gamma L_P(1+L_\pi)}$ \citep[Theorem 1]{rachelson2010locality}. 
This property can be useful to prove the Lipschitzness of the $Q$ function.

Additionally, in the case of delays, the smoothness of trajectories (sequence of consecutive states and actions) is a key factor.
Intuitively, smoother trajectories make the current unknown state more predictable.
Therefore, we consider the concept of \emph{time-Lipschitzness}, introduced by \citet{metelli2020control}. 

\begin{defi}[Time-Lipschitz MDP]
An MDP is said to be $L_T$-Time Lipschitz Continuous ($L_T$-TLC) if, $\forall s,a\in \Ss\times\As$
\begin{align*}
    \wass(p(\cdot|s,a), \delta_s)\le L_T,
\end{align*}
where $\delta_s$ is the Dirac distribution with mass on $s$.
\end{defi}

\section{Problem definition}
\label{sec:delays}
A delayed MDP (DMDP) stems from an MDP endowed with a sequence of variables $(\delay_t)_{t\in\mathbb{N}}$ corresponding to the delay at each step of the sequential process. The delay can affect the state observation, which implies that the agent has no access to the current state but only to a state visited $\delay_t$ steps before.
Affecting the action execution, the delay implies that the agent must select an action that will be executed $\delay_t$ steps from now. 
Lastly, reward collection delays may raise credit assignment issues and are outside the scope of this paper.  In any case, the DMDP violates the Markov assumption since the next observed state-reward couple does not depend only on the currently observable state and the chosen action.
In the literature, the delay is usually assumed to be a Markovian process, that is $\delay_t\sim P(\cdot\vert \delay_{t-1},s_{t-1},a_{t-1})$. Note that this definition includes \textit{state dependent delays} when $\delay_t\sim P(\cdot\vert s_{t-1})$, \textit{Markov chain delays} when $\delay_t\sim P(\cdot\vert \delay_{t-1})$ and \textit{stochastic delays} when $(\delay_t)_{t\in\mathbb{N}}$ are i.i.d. 

In this work we consider \textit{constant delays}, denoting the delay with the symbol $\delay$. When the delay is constant, the action execution delay and the state observation one are equivalent~\citep{katsikopoulos2003markov}, thus, it is sufficient to consider only the state observation delay. Furthermore, following \citet{katsikopoulos2003markov}, we consider a reward collection delay equal to the state observation delay so as not to collect a reward on a yet unobserved state, which could result in some form of partial state information.
Finally, we assume that the delay is known to the agent, placing ourselves in the \textit{non-anonymous} delay framework.

Within this reduced framework, it is possible to introduce an important concept of DMDPs, the \textit{augmented state}. Given the last observed state $s$ and the sequence of actions $(a_1,\dots,a_d)$ which have been taken since then, but whose outcome has not yet been observed, the agent can construct an augmented state, i.e., a new state in $\Xs=\Ss\times\As^{\delay}$ which casts the DMDP into an MDP \citep{bertsekas1987dynamic,altman1992closed}.
Said alternatively, the augmented state contains all the information the agent needs to learn the optimal policy in the DMDP. 
From the augmented state, we can gather information on the current state.
This information can be summarized by the \emph{belief}, the probability distribution of the current unknown state $s$ given the augmented state $x$ as $b(s\vert x)$. 
More explicitly, given $x=(s_1,a_1,\dots a_\delay)$, one has 
\begin{align*}
    b(s\vert x)=\int_{S^{\delay-1}}p(s|s_\delay,a_\delay)\prod_{i=2}^\delay p(s_i|s_{i-1},a_{i-1})ds_{i}.
\end{align*}
The delayed reward collected for playing action $a$ on $x$ is given by $\widetilde r(x,a)=\E_{\substack{s\sim b(\cdot\vert x)}}\left[r(s,a)\right]$.
To complete the DMDP framework, we define $\mudel$, the initial augmented state distribution. It samples the state contained in the augmented state under $\mu$ and samples the first action sequence under a distribution whose choice depends on the environment. We consider a uniform distribution on $\As$.
\section{Related works}
\label{sec:related}

The first proposed solution to the problem of delays is to use regular RL algorithms on the augmented-state MDP. 
Although the optimal delayed policy could potentially be obtained, this approach is affected by the exponential growth of the augmented state space, which becomes $|\Ss||\As|^\delay$ \citep{walsh2009learning} and is a source of the curse of dimensionality that is harmful in practice. 
Nonetheless, recent work by \citet{bouteiller2020reinforcement} revisits this approach and propose a clever way to resample trajectories without interacting with the environment by populating the augmented state with actions from a different policy, greatly improving the sample efficiency. They propose an algorithm, Delay-Correcting Actor-Critic (DCAC), which builds on SAC \citep{haarnoja2018soft} using the aforementioned resampling idea. DCAC has great experimental results and is sample efficient by design.

A second line of research focuses on memoryless policies, inspired by the partially observable MDP literature. It ignores the action queue to act according to the last observed state only. 
However, the delay can still be taken into account as in dSARSA \citep{schuitema2010control}, a modified version of SARSA \citep{sutton2018reinforcement} which accounts for the delay during its update. 
Indeed, SARSA would credit the reward collected for applying action $a$ on the augmented state $x$, containing the last observed state $s$, to the pair $(s,a)$. Instead, dSARSA proposes to credit $(s,a_1)$, where $a_1$ is the oldest action stored in $x$, the action actually applied on $s$. Despite being memoryless, dSARSA achieves great performances in practice.

Finally, the most common line of research, the model-based approach, relies on computing statistics on the current state which are then used to select an action. The name model-based comes from the fact that those solutions usually learn a model of the environment to predict the current state, by simulating the effect of the actions stored in the augmented state on the last observed state. 
\citet{walsh2009learning} learn the transition as a deterministic mapping so as to predict the most probable state, before selecting actions based on it.
\citet{derman2021acting} and \citet{firoiu2018human} propose a similar approach by learning the transitions with feed-forward and recurrent neural networks, respectively.
\citet{agarwal2021blind} estimate the transition probabilities and the undelayed $Q$ function to select the action that gives the maximum $Q$ under the estimated distribution of the current state. 
\citet{chen2021delay} use a particle-based approach to produce potential outcomes for the current state and, interestingly, extend the predictions to collect better value estimates.
\citet{liotet2021learning} propose D-TRPO which learns a vectorial encoding of the belief of the current state itself which is then used as an input to the policy, the latter being trained with TRPO \citep{schulman2015trust}. The authors also propose another algorithm, L2-TRPO which, instead of the belief, learns the expected current state by minimizing the predicted and the real state under the $l^2$-norm.

While most of these works assume that the delay is fixed, some consider the problem of stochastic delays \citep{bouteiller2020reinforcement,derman2021acting,agarwal2021blind}. Only one of them considers the case of non-integer delays \citep{schuitema2010control}.


\section{Imitation Learning for Delays}
\label{sec:imitation}
Our proposed approach is motivated by the limitations of two lines of research from the literature. 
Augmented approaches are affected by the curse of dimensionality that hinders the learning process, while model-based approaches require carefully designed models of the state transitions and usually involve a computational burden.
Instead, we propose to learn a mapping from augmented state directly to undelayed expert actions, facilitating the learning process as opposed to augmented approaches and by removing explicit approximation of transitions as opposed to model-based approaches. 
Our approach, however, implies that learning is split into two sub-problems: learning an expert undelayed policy and then imitating this policy in a DMDP.

\subsection{Imitation Learning}
It is usually easier to learn a behavior from demonstrations than learning from scratch using standard RL techniques. Imitation learning aims at learning a policy by mimicking the actions of an expert,
bridging the gap between RL and \textit{supervised learning}. Obviously, it requires that one can collect examples of an expert's behavior to learn from. 
For an expert policy $\pi_E$, most imitation learning approaches aim at finding a policy $\pi$ that minimizes $\E_{s\sim d_\mu^{\pi_E}}[l(s,\pi)]$~\citep{ross2011reduction} where $l(s,\pi)$ is a loss designed to make $\pi$ closer to $\pi_E$. Note that this objective is defined under the state distribution induced by $\pi_E$.
This can easily be problematic as, whenever the learner makes an error, it could end up in a state where its knowledge of the expert's behavior is poor and therefore errors could accumulate.  
Indeed, it has been shown that the error made by the learner potentially propagates as the squared effective horizon as shown in \citep[Theorem 1]{xu2020error}.
This is consistent with other bounds found in the literature depending on $H^2$ in the finite horizon setting \citep[Theorem 2.1]{ross2010efficient}.

One successful solution to this problem is \textit{dataset aggregation} as proposed by \citet{ross2011reduction} in their DA\textsc{gger} algorithm. The idea is to sample new data under the learned policy and query the expert on those new samples in order to match the learner's state distribution.  DA\textsc{gger} recursively builds a dataset $\mathcal{D}$ by sampling trajectories under policy $\pi_i = \beta_i \pi_E + (1-\beta_i) \hat{\pi}_i$ obtained from a $\beta_i$-weighted mixture of the expert policy and the previously imitated policy $\hat{\pi}_i$. 
One then queries the expert's policy on the states encountered in these trajectories and adds those tuples $(s,\pi_E(s))$ to $\mathcal{D}$. Finally, a new imitated policy $\hat{\pi}_{i+1}$ is trained on $\mathcal{D}$. 
The sequence $(\beta_i)_{i\in[\![1,N]\!]}$ is such that $\beta_1=1$, so as to sample initially only from $\pi_E$ and $\beta_N=0$ to sample only from the imitated policy in the end. 

\subsection{Duality of Trajectories}
Once sampled, the trajectories, either from a DMDP or its underlying MDP, can be interpreted in both processes when the delay is an integer number of steps. 
In a DMDP, the current state will eventually be observed by a delayed agent.
In an MDP, the trajectories can be re-organized to simulate the effect of a delay. 
In particular, this means that one can collect trajectories with an undelayed environment and sample either from an undelayed policy or a delayed policy (by creating a synthetic augmented state). 
This is exactly what is required to adapt DA\textsc{gger} to imitate an undelayed expert with a delayed learner.

\begin{algorithm}[t] 
\caption{Delayed Imitation with DA\textsc{gger} (DIDA)}\label{algo:dida}
\textbf{Inputs}
undelayed environment $\mathcal E$, undelayed expert $\pi_E$, $\beta$ routine, number of steps $N$, empty dataset $\mathcal{D}$.\\
\textbf{Outputs}: delayed policy $\pi_I$
\begin{algorithmic}[1]
    \FOR{$\beta_i$ in $\beta$-routine}
        \FOR{$j$ in $\{1,\dots,N\}$}
            \IF{New episode}
                \STATE Initialize state buffer $( s_{1},s_{2}, \dots, s_{\delay})$ \\and action buffer $( a_{1},a_{2} \dots, a_{\delay-1})$
            \ENDIF
            \STATE Sample $a_E\sim\pi_E(\cdot\vert {{s}}_{\delay})$, set $a=a_E$ \label{ope:undelayed_sample}
            \IF{Random $u \sim U([0,1])\geq \beta_i$}
                \STATE Overwrite ${ a \sim\pi_I(\cdot\vert [{s}_{1},a_{1}, \dots, a_{\delay-1}])}$ 
            \ENDIF
            \STATE Aggregate dataset:
            \\$\qquad \mathcal{D}\leftarrow \mathcal{D} \cup ([{s}_{1}, a_{1}, \dots, a_{\delay-1}],a_E)$
            \STATE Apply $a$ in $\mathcal E$ and get new state $s$ 
            \STATE Update buffers:
            \\
            $\qquad ({s}_1,\dots,{s}_\delay) \leftarrow ({s}_2,\dots,{s}_\delay,s)$
            \\
            $\qquad  ({a}_1,\dots,{a}_{\delay-1}) \leftarrow ({a}_2,\dots,{a}_{\delay-1},a)$
            \ENDFOR
        \STATE Train $\pi_I$ on $\mathcal{D}$ 
    \ENDFOR
\end{algorithmic}
\end{algorithm} 

\subsection{Imitating an Undelayed Policy}

We follow the learning scheme of DA\textsc{gger} with the slight difference that, if the expert is queried, then the current state is fed to $\pi_E$ while if the imitator policy is queried, an augmented state is built from the past samples, considering the state $\delay$-steps before the current state and the sequence of actions taken since then. This implies that a buffer of the latest states and actions has to be built. 
We present our approach, which we call Delayed Imitation with DA\textsc{gger} (DIDA), in \Cref{algo:dida}.
In practice there is no need to store each augmented state in the dataset $\mathcal{D}$ since most of the actions contained inside one are also contained in others. Therefore, only trajectories of state and action can be stored, from which augmented states are recreated during the training of $\pi_I$.

\textit{}{What will the policy learned by DIDA be in practice?}~Given an augmented state $x$, DIDA learns to replicate the action taken by the expert on the current state $s$, unknown to the agent. However, the same augmented state can lead to different current states, which is summarized in the belief $b(s\vert x)$. Therefore, DIDA learns the following policy
\begin{align}
    \pidida(a\vert x) = \int_{\Ss} b(s\vert x) \pi_E (a\vert s) ds.\label{eq:belief_pol}
\end{align}
The learned policy is therefore similar to the policies from model-based approaches, and, for this reason, may yield sub-optimal policies in some MDPs \citep[Proposition VI.1.]{liotet2021learning}.
In practice, the class of functions of the imitated policy $\pi_I$ and the loss chosen for training in step 15 of \Cref{algo:dida} may slightly modify the policy learned by DIDA. For instance, a deterministic $\pi_I$ would naturally forbid to learn the distribution given in \Cref{eq:belief_pol}.
This is discussed in \Cref{app:dida_policy}.

\subsection{Extension to non integer delays}
\label{subsec:non_int_delay}
We now suppose that the delay is non-integer, yet still constant. For simplicity, we assume $\delay\in (0,1)$ but the general case follows from similar considerations. We consider a $\delay$-delay in the action execution (the case of state observation is similar).

DMDP with non-integer delays can be viewed as the result of two interleaved MDPs, $\mathcal{M}$ with time indexes $t,t+1,\dots$ and $\mathcal{M}_\delay$ with indexes $t+\delay,t+\delay+1,\dots$. Those two discrete MDPs stem from a single continuous process, of which we observe only some fixed time steps, similarly to \citep{sutton1999between}. They share the same transition and reward functions.
A delayed agent would see states from $\mathcal{M}$ while executing actions on states of $\mathcal{M}_\delta$. 
In practice, the agent taking action $a_t$ seeing state $s_t$ would collect a reward $r(s_{t+\delay},a_t)$.
The transition probabilities are also affected. 
We define $b_\delay(s_{t+\delay}\vert s_t,a)$ the probability of reaching $s_{t+\delay}$ from $s_t$ when action $a$ is applied during time $\delay$ and $b_{1-\delay}(s_{t}\vert s_{t-1+\delay},a)$ the probability of reaching $s_{t+\delay}$ from $s_t$ when action $a$ is applied during time $1-\delay$.
To make the definition consistent with the regular MDP, those probabilities must satisfy that, for all $(s,a,s')\in\Ss\times\As\times\Ss$
\begin{align}
    p(s'\vert s,a)=\int_{\Ss}   b_{1-\delay}(s'\vert z,a)b_\delay(z|s,a)~dz.\label{eq:def_p_non_int}
\end{align}
Clearly, even for $\delay\in (0,1)$, an augmented state $x_t = (s_t,a_{t-1})\in \Ss \times \As \eqqcolon \mathcal X$ is needed in order not to lose information about the state $s_{t+\delay}\sim b_\delay(\cdot\vert s_t,a_{t-1})$. 
The DMDP with augmented state can again be cast into an MDP as in \citet{bertsekas1987dynamic,altman1992closed}, where the new transition is defined for $x_t = (s_t,a_{t-1}),  x_{t+1} = (s_{t+1},a_{t}) \in\Xs$,
{\thinmuskip=0mu
\medmuskip=0mu
\thickmuskip=0mu
\begin{align*}
    \tilde p(x_{t+1}|x_t, a)\coloneqq \delta_{a}(a_{t})\int_{\Ss} b_{1-\delay}(s_{t+1}|z,a)b_\delay(z|s_{t},a_{t-1})   ~dz,
\end{align*}%
}%
where the term $\delta_{a}(a_{t})$ ensures that the new extended state contains the action that has been applied on $x_t$. 
For delays greater than 1, one needs to consider the augmented state in the space $\Ss \times \As^{\lceil\delay\rceil}$ and the previous considerations hold by first considering the integer part of the delay and then its remaining non-integer part.
In this setting, we propose to use DIDA by learning an undelayed policy in $\mathcal{M}_\delay$ and imitating it by building an augmented state from the states in $\mathcal{M}$.

\section{Theoretical analysis of the approach}
\label{sec:theoretical}

We will now provide a theoretical analysis of the approach proposed above. 
The role of this analysis is twofold. First, it gives insights into which expert undelayed policy is best suited to be imitated in a DMDP. Secondly, it provides general results on the value functions bounds between DMDPs and MDPs, when the latter has guarantees of smoothness, setting aside pathological counterexamples such as in \citep[Proposition VI.1.]{liotet2021learning} while remaining realistic.
To compare the performance of delayed and undelayed policies, we have to compare the corresponding state value functions, which is non trivial, since they live on two different spaces ($\Ss$ and $\Xs=\Ss\times \As^\delay$).

Different approaches were proposed to address this issue. In \citep[Theorem 3]{walsh2009learning}, assuming a finite MDP with \textit{mildly stochastic} transitions, that is, there exists $\epsilon$ such that, $\forall (s,a)\in\Ss\times\As, \exists s', p(s'\vert s,a)\geq 1-\epsilon$, then, for some undelayed policy $\pi$ one can bound the value function in the deterministic approximation of the MDP, $\widetilde{V}^{\pi}$ with respect to the value function in the real MDP, $V^{\pi}$ as $\lVert \widetilde{V}^{\pi}-V^{\pi}\rVert_{\infty}\leq \frac{\gamma\epsilon R_{\max}}{(1-\gamma)^2}$, where $R_{\max}$ is a bound on the reward. 
The assumptions by \citet{walsh2009learning} are quite strong and the bound grows quadratically with the effective time horizon.
Another approach is proposed by \citet[Theorem 1]{agarwal2021blind}, who compare the delayed value function $V^{\pidel}$ to $\E_{s\sim b(\cdot\vert x)}[V^{\pi}(s)]$, which corresponds to the expected value function of the undelayed policy averaged on the current unknown state given some augmented state. However, the authors make no assumptions about smoothness. 

Instead, we base our analysis on smoothness assumptions to provide our main result on the difference in performance between delayed and undelayed policies in \Cref{th:perf_diff_bound}. To obtain this result, we must first derive a delayed version of the performance difference lemma~\citep{kakade2002approximately}. Its proof, as for all other results in this section, is given in \Cref{app:proofs} and applies to any couple of delayed and undelayed policies.
Note that these results hold for either integer or non-integer constant delays. For simplicity, we state the results with belief $b$ but $b_\delay$ is intended if the delay is non-integer.




\begin{restatable}{lem}{perfdifflem}[Delayed Performance Difference Lemma]
\label{lem:perfdimlem}
    Consider an undelayed policy $\pi_E$ and a $\delay$-delayed policy $\pidel$, with $\delay\in\mathbb R_{\ge0}$. 
    Then, for any $x\in\mathcal{X}$,
    \begin{align*}
        \E_{s\sim b(\cdot\vert x)}&[V^{\pi_E}(s)] - V^{\pidel}(x) = \frac{1}{1-\gamma}
        \\
        &\E_{x'\sim d^{\pidel}_x} \left[ \E_{s\sim b(\cdot\vert x')}[V^{\pi_E}(s)] -  \E_{\substack{s\sim b(\cdot\vert x')\\a\sim \pidel(\cdot\vert x')}}[Q^{\pi_E}(s,a)] \right].
    \end{align*}
\end{restatable}
We can then leverage the previous result to obtain a valuable result for DMDPs, which holds for delayed policies of the form of \Cref{eq:belief_pol}.
\begin{restatable}{thm}{perfdiffbound}
\label{th:perf_diff_bound}
    Consider an $(L_P,L_r)$-LC MDP and a $L_\pi$-LC undelayed policy $\pi_E$,  such that $Q^{\pi_E}$ is $L_Q$ -L.C.\footnote{In fact, only Lipschizness in the second argument is necessary (see proof).}.
    Let $\pidida$ be the $\delay$-delayed policy defined as in \Cref{eq:belief_pol}, with $\delay\in\mathbb R_{\ge0}$. 
    Then, for any $x\in\mathcal{X}$,
    \begin{align*}
        \E_{s\sim b(\cdot\vert x)}&[V^{\pi_E}(s)] - V^{\pidida}(x) \leq \frac{L_Q L_\pi}{1-\gamma} \sigma_b^x,
    \end{align*}
    where $\sigma_b^x = \E_{\substack{x'\sim d^{\pidida}_x\\ s,s'\sim b(\cdot\vert x')}}\left[ \dist_{\Ss}(s,s') \right]$.
\end{restatable}

However, this result seems difficult to grasp because of its dependence on the term $\sigma_b^x$.
We suggest two ways to further bound this term.
The first involves the time-Lipschitzness assumption of the MDP and yields \Cref{th:perf_diff_bound_tlc}.

\begin{restatable}{cor}{perfdiffboundtlc}
\label{th:perf_diff_bound_tlc}
    Under the assumptions of \Cref{th:perf_diff_bound}, adding that the MDP is $L_T$-TLC, then, for any $x\in\mathcal{X}$,
    \begin{align*}
        \E_{s\sim b(\cdot\vert x)}&[V^{\pi_E}(s)] - V^{\pidida}(x) \leq \frac{2 \delay L_T L_Q L_\pi}{1-\gamma}.
    \end{align*}
\end{restatable}

This first result clearly highlights the linear dependence on the delay $\delay$.
However, the bound does not vanish (as expected) when the MDP is deterministic, but this is verified by a second result.  
This second result assumes a state space in $\mathbb{R}^n$ equipped with the Euclidean norm and yields \Cref{th:perf_diff_bound_eucl}.

\begin{restatable}{cor}{perfdiffboundeucl}
\label{th:perf_diff_bound_eucl}
    Under the assumptions of \Cref{th:perf_diff_bound} adding that $\Ss\subset\mathbb{R}^n$ is equipped with the Euclidean norm.
    Then, for any $x\in\mathcal{X}$,
    \begin{align*}
        \E_{s\sim b(\cdot\vert x)}[V^{\pi_E}(s)] - V^{\pidida}(x) &\leq
        \frac{2 L_Q L_\pi}{1-\gamma}
        \\&\E_{x'\sim d_x^{ \pidida}(\cdot)}\left[\sqrt{ \Var_{s\sim b(\cdot|x')}(s|x')}\right].
    \end{align*}
\end{restatable}

Interestingly, we show that this second corollary matches a theoretical lower bound when the expert policy is optimal. 
We provide this lower bound in \Cref{th:lb_tight}, which shows that a too irregular expert policy (with high Lipschitz constant) provides weaker guarantees. 

\begin{restatable}{thm}{lbtight}
\label{th:lb_tight}
For every $L_\pi>0$, $L_{Q}>0$, there exists an MDP such that the optimal policy is $L_\pi$-LC, its state action value function is $L_{Q}$-LC in the second argument, but for any $\delay$-delayed policy $\pidel$, with $\delay\in\mathbb R_{\ge0}$, and any $x\in\Xs$
    \begin{align*}
        \E_{s\sim b(\cdot\vert x)}[V^*(s)] - V^{\pidel}(x) &\geq
        \frac{\sqrt 2}{\sqrt \pi}\frac{ L_{Q}L_\pi}{1-\gamma}
        \\
        &\E_{x'\sim d_x^{\tilde \pi}(\cdot)}\left[\sqrt{ \Var_{s\sim b(\cdot|x')}(s|x')}\right],
    \end{align*}
    where $V^*$ is the value function of the optimal undelayed policy.
\end{restatable}

We provide an alternative way to derive bounds in performance in \Cref{app:bound_state_dist}, which provide slightly different results as discussed in \Cref{subsec:disc_bounds}.

We have bounded the performance of our perfectly imitated delayed policy $\pidida$ with respect to the undelayed expert $\pi_E$. However, two additional sources of performance loss have to be taken into account. First, the expert $\pi_E$ may be sub-optimal in the undelayed MDP. Second, the imitated policy $\pi_I$ may not learn exactly $\pidida$. 


These theoretical results highlight two important trade-offs in practice. If the expert policy is smoother than the optimal undelayed policy, then we might miss out on some opportunities, but the delayed policy is likely to be more similar to the expert one, according to \Cref{th:perf_diff_bound}. The second trade-off concerns noisier policies.  For them, the imitation step is likely to be easier, as it provides examples of how to recover from bad decisions~\citep{laskey2017dart}. 
Therefore, our imitated policy $\pi_I$ is likely to be more similar to $\pidida$. However, this may decrease the performance of the expert compared to the optimal undelayed policy.

\begin{figure*}[t]
\centering
    \begin{subfigure}{0.33\textwidth}
        \centering
        \includegraphics[width=\linewidth]{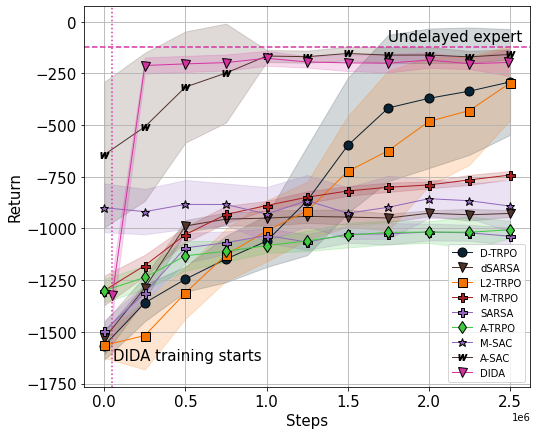}
        \caption{Pendulum.}
        \label{fig:pendulum}
    \end{subfigure}%
\hfill
    \begin{subfigure}{0.33\textwidth}
        \centering
        \includegraphics[width=\linewidth]{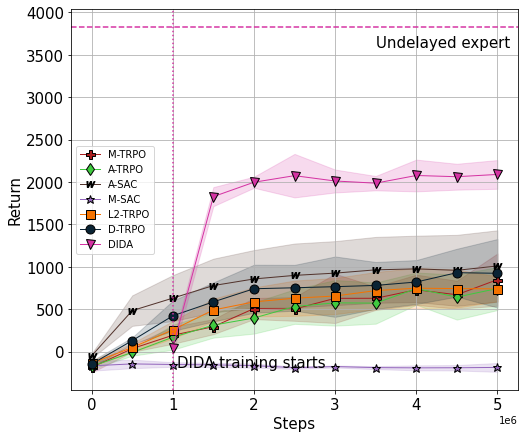}
        \caption{HalfCheetah.}
        \label{fig:halfcheetah}
    \end{subfigure}%
\hfill
    \begin{subfigure}{0.33\textwidth}
        \centering
        \includegraphics[width=\linewidth]{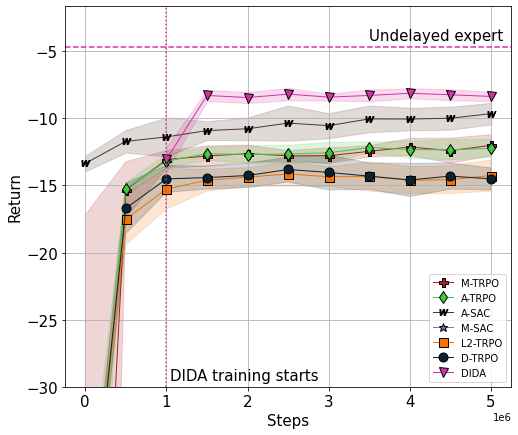}
        \caption{Reacher.}
        \label{fig:reacher}
    \end{subfigure}%
\caption{For a 5-steps delay, mean return and one standard deviation (shaded) as a function of the number of steps sampled from the environment (10 seeds).}
\end{figure*}

\section{Experiments}
\label{sec:experiments}

\subsection{Setting}
As we have seen in the theoretical analysis, a smoother expert is beneficial for the performance bound of the imitated delayed policy. Therefore, in the following experiments, we consider expert policies learned with SAC. As reported in an extensive study about smooth policies \citep{mysore2021regularizing}, the entropy-maximization framework of SAC is able to learn a smooth policy even without additional forms of regularization.
To avoid ever-growing memory by storing all samples in the buffer as done in \Cref{algo:dida}, we use a maximum buffer size of 10 iterations for DIDA and overwrite the oldest iteration samples when this buffer is full. 
As suggested by \citet{ross2011reduction}, we use $\beta_1=1, \beta_{i\ge 2}=0$ as mixture weights for the sampling policy.
The policy for DIDA is a simple feed-forward neural network. More details and all hyper-parameters are reported in \Cref{app:hyper_param}.

We will test DIDA, along with some baselines from the state of the art, on the following environments.

\textbf{Pendulum}~~The task of the agent is to rotate a pendulum upward. It is a classic experiment in delayed RL as delays are highly impacting performance due to unstable equilibrium in the upward position. We use the version from the library \texttt{gym}~\citep{brockman2016gym}.

\textbf{Mujoco}~~Continuous robotic locomotion control tasks realized with an advanced physics simulator from the library \texttt{mujoco}~\citep{todorov2012mujoco}. Here the main difficulty lies in the complex dynamics and in the large state and action spaces. Among the possible environments, we consider the ones that are most affected by delays, namely Walker2d, HalfCheetah, Reacher, and Swimmer.

\textbf{Trading}~~The agent trades the EUR-USD (€/\$) currency pair on a minute-by-minute basis and can either \emph{buy}, \emph{sell} or stay \emph{flat} against a fixed amount of USD, following the framework of \citet{bisi2020foreign} and \citet{riva2021learning}. 
We assume trading is without fees, but we do take the spread into account. 
To this setting, we add a delay of 10 seconds to the action execution.
In this environment, we leverage the knowledge of an expert which is a policy trained on years 2016-2017 by Fitted Q-Iteration~\citep[FQI][]{ernst2005tree} with XGBoost~\citep{chen2016xgboost} as a regressor for the $Q$ function. 
Only for this task, we use Extra Trees~\cite{geurts2006extremely} as policy for DIDA.

The baselines for comparison with our algorithm include a memoryless and an augmented version of TRPO (M-TRPO and A-TRPO respectively),  D-TRPO and L2-TRPO~\citep{liotet2021learning}, SARSA~\citep{sutton2018reinforcement} and dSARSA~\citep{schuitema2010control}. The last two algorithms involve state discretization and are thus tested on pendulum only.
We consider also augmented SAC (A-SAC), considered also by \citet{bouteiller2020reinforcement}, and memoryless SAC (M-SAC). 
Although SAC can be trained at every step as we do for Pendulum, we restrict training to every 50 steps on Mujoco to speed up the procedure and reduce memory usage.
We have also considered adding DCAC~\citep{bouteiller2020reinforcement} but for computational reasons, we have decided not to include it. Early experimental results showed that its running time was more than 50 times the one of DIDA. 
For a fair comparison to the baselines, which learn a policy from scratch, we include the training steps of the expert in the step count of DIDA, as indicated by the vertical dotted line in the figures.

\subsection{Results}


\begin{figure*}[t]
\centering
\begin{minipage}{.66\textwidth}
  \centering
  \begin{subfigure}{.49\textwidth}
    \centering
    \includegraphics[width=\linewidth]{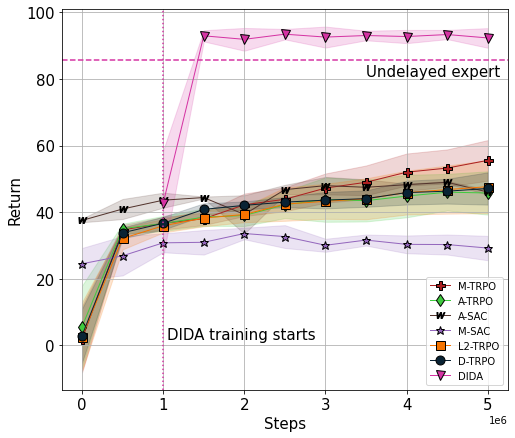}
    \caption{Swimmer.}
    \label{fig:swimmer}
\end{subfigure}%
\hfill
\begin{subfigure}{.49\textwidth}
  \centering
    \includegraphics[width=\linewidth]{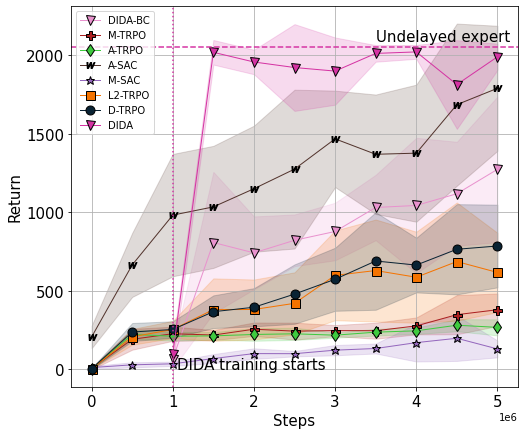}
    \caption{Walker2d.}
    \label{fig:walker}
\end{subfigure}
\setcounter{figure}{1}
\captionof{figure}{For a 5-steps delay, mean return and one standard deviation (shaded) as a function of the number of steps sampled from the environment (10 seeds).}
\end{minipage}%
\hfill
\begin{minipage}{.33\textwidth}
  \centering
    \includegraphics[width=\linewidth]{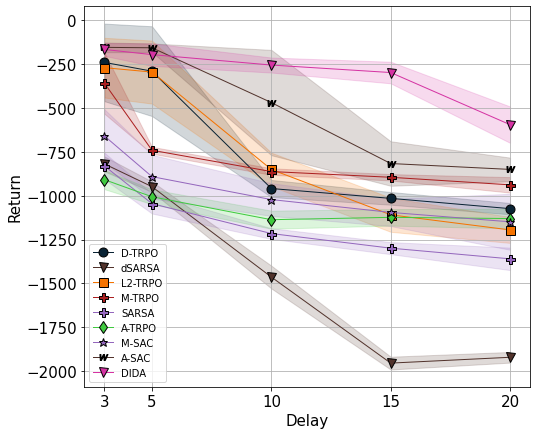}
    \captionof{figure}{Mean return and its standard deviation (shaded) as a function delay (10 seeds).}
    \label{fig:pendulum_delay}
\end{minipage}
\end{figure*}

\begin{figure}[t]
\centering
    \includegraphics[width=0.8\linewidth]{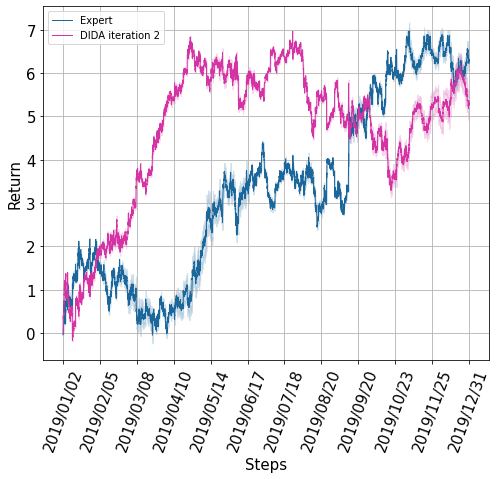}
        \caption{Evolution of the return of DIDA for the trading of the EUR-USD pair in 2019. Performance is computed in percentage w.r.t. the invested amount.}
        \label{fig:trading}
\end{figure}

As we can see from the results on pendulum and mujoco, \Cref{fig:pendulum,fig:halfcheetah,fig:reacher,fig:swimmer,fig:walker}, DIDA is able to converge much faster than the baselines, in any environment with the exception of A-SAC on the pendulum environment. 
In less than half a million steps on mujoco, and 250.000 steps on pendulum, DIDA almost reaches its final performance. 
We note that, in HalfCheetah and Reacher, DIDA, although the best delayed algorithm, performs much worse than the expert. Surprisingly, in Swimmer, DIDA performs slightly better than the undelayed expert. All these phenomenons might actually be due to a single cause. 
In our implementation, to initialize the environment, a sequence of $\delay-1$ random actions are applied in an undelayed environment to sample a first delayed augmented state. Depending on the environment, this sequence could cause the agent to start in un-advantageous or advantageous states. For instance, in HalfCheetah, the random action have put the agent head-down when the the latter is first allowed to control the environment. It must thus first get back on its feet before starting to move. On the contrary, in a simpler environment like Swimmer, the initial random action queue might give some initial speed to the agent, yielding higher rewards at the beginning than its undelayed counterpart.

We provide another experiment on Pendulum where we study the robustness of DIDA, as compared to baselines, against an increase in the delay for fixed hyper-parameters. We report the final mean return per episode for different values of the delay in \Cref{fig:pendulum_delay}. Clearly, from all the baselines studied, DIDA is the most robust to the increase in the delay.

For the trading task, which is a batch-RL task since the training dataset is a fixed set of historical exchange rates, DIDA is prone to overfitting the expert policy on these examples. 
Therefore, after training several iterations of DIDA, we select the best iteration on the validation year 2018 and show the test performance in the year 2019 compared to the undelayed expert.
In our results, we consider two experts trained on two different seeds, but with the same hyper-parameters configuration. 
We then imitated each seed with DIDA.
The results, as shown in \Cref{fig:trading}, show the ability of DIDA to adapt to non-integer delays and maintain a positive return, which is not a simple task when taking the spread into account in trading the EUR-USD.
One may notice that the delayed policy is able to outperform the expert on the first period of the test. This could be explained by the fact that the expert undelayed policy may have overfitted the training set while the imitation learning of an undelayed policy acted as a regularization.
We provide an analysis on the policy learned by DIDA with respect to the expert in \Cref{subsec:further_exp}. 

Moreover, we provide in \Cref{subsec:further_exp} additional experiments on a stochastic version of pendulum and a study of the impact of a growing delay on the performance of DIDA.
\section{Conclusion}

In this paper, we explored the possibility of splitting delayed reinforcement learning into easier tasks, traditional undelayed reinforcement learning on the one hand, and imitation learning on the other one. 
We provided a theoretical analysis demonstrating bounds on the performance of a delayed policy compared to undelayed experts, both for integer or non-integer constant delays.
These bounds apply in our particular setting but are also of interest in general for delayed policies.
This guided us in the creation of our algorithm, DIDA, which learns a delayed policy by imitating an undelayed expert using DA\textsc{gger}. 
We have empirically shown that this idea, although rather simple, provides excellent results in practice, achieving high performance with remarkable sample efficiency and light computations.
We believe that our work paves the way for many possible generalizations, which include stochastic delays and particular situations in which an undelayed simulator is not available, but where an undelayed dataset can be artificially created from delayed trajectories in order to train an expert offline.

\bibliography{biblio}
\bibliographystyle{icml2021}

\normalsize
\appendix
\onecolumn
\section{General Results}

\subsection{Bounds involving the Wasserstein distance}
\begin{prop}
\label{pp:exp_wass}
    Let X,Y be two random variables on $\mathbb{R}$ with distribution $\pi_0,\pi_1$ respectively. Then,
    \begin{align*}
        \left\vert \E[X]-\E[Y]  \right\vert\leq \wass(\pi_0\Vert \pi_1).
    \end{align*}
\end{prop}
\begin{proof}
One has 
    \begin{align*}
        \E[X]-\E[Y] = \int_{\mathbb{R}} x(\pi_0(x)-\pi_1(x))dx\leq \wass(\pi_0\Vert \pi_1),
    \end{align*}
    since $x\mapsto x$ is 1-LC.
    The same holds for $\E[Y]-\E[X]$, since the Wasserstein distance is symmetric.
\end{proof}

The next result asserts that if one applies a $L$-LC function to two random variables, one gets two random variables with distribution whose Wasserstein distance is bounded by the original Wasserstein distance multiplied by a factor $L$.
\begin{prop}\label{pp:wass_lip_g}
    Let $g:\Omega\rightarrow \mathbb{R}$ be an $L$-LC function and $\pi_0,\pi_1$ two probability measures over the metric space $\Omega$.
    Note $g_{\pi}$ the distribution of the random variable $g(X)$ where $X$ is distributed according to $\pi$.
    Then,
    \begin{align*}
        \wass(g_{\pi_0}\Vert g_{\pi_1})\leq L \wass(\pi_0 \Vert \pi_1).
    \end{align*}
\end{prop}
\begin{proof}
    By definition of Wasserstein distance,
    \begin{align*}
        \wass(g_{\pi_0}\Vert g_{\pi_1}) 
        &= \sup_{\left\Vert f\right\Vert_L\leq 1}\left\vert\int_{\mathbb R} f(x)(g_{\pi_0}(x)-g_{\pi_1}(x)) dx\right\vert
        \\
        &=\sup_{\left\Vert f\right\Vert_L\leq 1}
        \left\vert\int_{\mathbb R} f(x)g_{\pi_0}(x) dx-\int_{\mathbb R} f(x)g_{\pi_1}(x) dx\right\vert.
    \end{align*}  
    We can then use the definitions of $g_{\pi_0}$ and $g_{\pi_1}$ to rewrite the previous formula in terms of expected values
    \begin{align*}
        \wass(g_{\pi_0} \Vert g_{\pi_1}) 
        &= \sup_{\left\Vert f\right\Vert_L\leq 1}
        \left\vert\int_{\mathbb R} f(g(x))\pi_0(x) dx-\int_{\mathbb R} f(g(x)) \pi_1(x) dx\right\vert
        \\
        &= \sup_{\left\Vert f\right\Vert_L\leq 1}
        \left\vert \int_{\mathbb R} f(g(x))(\pi_0(x) - \pi_1(x)) dx\right\vert.
    \end{align*}  
    Since $g$ is $L_g$-LC by assumption, the composition $f(g(x))$ is still $L_g$-LC, so 
    \begin{align*}
        \wass(g_{\pi_0}\Vert g_{\pi_1}) 
        &\leq L_g \wass(\pi_0 \Vert \pi_1).
    \end{align*}  
\end{proof}

\begin{prop}\label{pp:exp_q_bound}
    Consider an MDP with $\pi$ a policy such that its state-action value function is Lipschitz with constant $L_{Q}$ in the second argument, i.e. it satisfies, for all $s\in\Ss$ and $a,a' \in \As$
    \begin{align*}
    \left\vert Q^{\pi}(s,a)-Q^{\pi}(s,a')\right\vert\le L_{Q} \dist_\As(a,a'),
    \end{align*}
    then, for every couple of probability distributions $\eta,\nu$ over $\As$, one has that 
    \begin{align*}
        \left\vert \E_{\substack{X\sim \eta\\Y\sim \nu}}[Q^{\pi}(s,X)-Q^{\pi}(s,Y)] \right\vert 
        \leq L_{Q} \wass(\eta(\cdot)\Vert\nu(\cdot)).
    \end{align*}
\end{prop}
\begin{proof}
    We note $g_{\eta}$ and $g_{\nu}$ the respective distributions of $Q^{\pi}(s,X)$ and $Q^{\pi}(s,Y)$.
    First of all, we can apply \Cref{pp:exp_wass} to say that
    \begin{align*}
        \left\vert \E_{\substack{X\sim \eta\\Y\sim \nu}}[Q^{\pi}(s,X)-Q^{\pi}(s,Y)] \right\vert 
        \leq  \wass(g_{\eta}\Vert g_{\nu}).
    \end{align*}
    For a fixed $s\in\Ss$, the two random variables $Q^{\pi}(s,X)$ and $Q^{\pi}(s,Y)$ can be seen as the application of the $L_{Q}$-Lipschitz function
    $Q^{\pi}(s,\cdot):\ \As\to \mathbb R$
    to $X$ and $Y$, respectively.
    This satisfies the assumptions of \Cref{pp:wass_lip_g}, therefore 
    \begin{align*}
         \wass(Q^{\pi}(s,\eta(\cdot))\Vert Q^{\pi}(s,\nu(\cdot)))
         \leq L_{Q} \wass(\eta(\cdot)\Vert \nu(\cdot)).
    \end{align*}
\end{proof}

\begin{prop}
\label{pp:lip_g}
    Consider an $L_P$ transition function $p$ and an $L_\pi$ policy $\pi$ in some MDP $\mathcal{M}$. 
    Then, for any $f:~S\to~\mathbb R$ which is $1$-LC, we have that the function
    $g:~S\to~\mathbb R$ given by
    \begin{align*}
        g(s)\coloneqq \int_{\Ss} f(s')\int_A p(s'\vert s,a)\pi(a\vert s)~da~ds'
    \end{align*}
    is Lipschitz with constant $L_p(1+L_\pi)$
\end{prop}
\begin{proof}
    Let $s,z\in S$, one has
    \begin{align}
        \lvert g(s)-g(z) \rvert &=\left\vert \int_{\Ss} f(s')\int_A p(s'\vert s,a)\pi(a\vert s)-p(s'\vert z,a)\pi(a\vert z)~da~ds'\right\vert \nonumber
        \\
        &\leq \left\vert \int_{\Ss} f(s')\int_A p(s'\vert s,a)\left(\pi(a\vert s)-\pi(a\vert z)\right)~da~ds'\right\vert \nonumber
        \\
        &\quad + \left\vert \int_{\Ss} f(s')\int_A \left(p(s'\vert s,a)-p(s'\vert z,a)\right)\pi(a\vert z)~da~ds'\right\vert \label{eq:tri_ineg_term}
        \\
        &\leq \underbrace{\left\vert \int_A \left(\pi(a\vert s)-\pi(a\vert z)\right)\int_{\Ss} f(s')p(s'\vert s,a) ~ds'~da\right\vert}_{A} \nonumber
        \\
        &\quad + \underbrace{\left\vert \int_A \pi(a\vert z) \int_{\Ss} f(s') \left(p(s'\vert s,a)-p(s'\vert z,a)\right)~ds'~da\right\vert}_{B} \label{eq:fubini_a_s},
    \end{align}
    where we add and remove the quantity $p(s'\vert s,a)\pi(a\vert z)$ in \Cref{eq:tri_ineg_term} and use Fubini's theorem in \Cref{eq:fubini_a_s}.
    
    By Lipschitzness of $p$, we have that $a\mapsto \int_{\Ss} f(s')p(s'\vert s,a)~ds'$
    is $L_p$-LC. Thus, 
    \begin{align*}
        A &=  L_P \left\vert \int_A \left(\pi(a\vert s)-\pi(a\vert z)\right)\frac{\int_{\Ss} f(s')p(s'\vert s,a) ~ds'}{L_P}~da\right\vert
        \\
        &\leq L_\pi L_P \dist_{\Ss}(s,z)
    \end{align*}
    
    For the second term, again, by Lipschitzness of $p$, we have
    \begin{align*}
        B &\leq \left\vert \int_A L_P d_{\Ss}(s,z) \pi(a\vert z) ~da\right\vert
        \\
        &= L_P d_{\Ss}(s,z).
    \end{align*}
    
    Overall, 
    \begin{align*}
        \lvert g(s)-g(z) \rvert &\leq L_p(L_\pi+1)\dist_{\Ss}(s,z).
    \end{align*}
\end{proof}

\subsection{Bounding $\sigma_b^{\rho}$}
We provide two bounds for $\sigma_b^\rho = \E_{\substack{x'\sim d^{\pidel}_\rho\\ s,s'\sim b(\cdot\vert x')}}\left[ \dist_{\Ss}(s,s') \right]$ with $\rho$ a distribution on $\Ss$. The first uses the assumption that the state space in is $\mathbb{R}^n$ and is equipped with the Euclidean norm while the second assumes that the MDP is TLC.
\begin{lem}[Euclidean bound]
\label{lem:bound_sigma_eucl}
    Consider an MDP such that $\Ss\subset\mathbb{R}^n$ is equipped with the Euclidean norm. Then one has
    \begin{align*}
        \sigma_b^\rho \leq \E_{x'\sim d_\rho^{\tilde \pi}(\cdot)}\left[\sqrt{ \Var_{s\sim b(\cdot|x')}(s|x')}\right].
    \end{align*}
\end{lem}
\begin{proof}
    We derive the following results which intermediate steps are detailed after.
    \begin{align}
        \sigma_b^\rho 
        &= \E_{\substack{x'\sim d^{\pidel}_\rho\\ s,s'\sim b(\cdot\vert x')}}\left[ \dist_{\Ss}(s,s') \right]\nonumber
        \\
        &= \E_{\substack{x'\sim d^{\pidel}_\rho\\ s,s'\sim b(\cdot\vert x')}}\left[ \sqrt{(s'-s)^2} \right]\label{eq:def_euclid}
        \\
        &= \E_{x'\sim d^{\pidel}_\rho}\sqrt{\E_{s,s'\sim b(\cdot\vert x')} \left[ (s'-s)^2 \right]}\label{eq:jensen_var}.
    \end{align}
    \Cref{eq:def_euclid} follows from the definition of the Euclidean norm and \Cref{eq:jensen_var} is obtained by applying Jensen's inequality. 
    To conclude, since $s'$ and $s$ are i.i.d., one has that 
    $\E_{s,s'\sim b(\cdot\vert x')} \left[ (s'-s)^2 \right] =
    2\mathbb Var_{s\sim b(\cdot|x')}[s]$.
\end{proof}

The following proposition is involved in the proof of the bound of $\sigma_b^\rho$ when the MDP is TLC.
\begin{prop}
\label{pp:tlc_delay}
    Consider an $L_T$-TLC MDP. Consider any augmented state $x=(s_1,a_1,\dots,a_\delay)\in\Ss\times\As^\delay$ for a given $\delay\in\mathbb{N}$.
    Then 
    \begin{align*}
        \wass\left(b(\cdot\vert x)\Vert \delta_{s_1} \right)\leq \delay L_T
    \end{align*}
\end{prop}
\begin{proof}
    We proceed by induction. The case case d=0 is true since the current state is known exactly without delay. 
    The case d=1 is true by the assumption of $L_T$-TLC.
    Assume that the statement is true for $d\in\mathbb{N}$, then
    \begin{align}
        \wass\left(b(\cdot\vert x)\Vert \delta_{s_1} \right) 
        &= \sup_{\left\Vert f\right\Vert_L\leq 1}\left\vert\int_{\Ss} f(s')\left( b(s' \vert x) -\delta_{s_1}(s')\right)ds' \right\vert\nonumber
        \\
        &= \sup_{\left\Vert f\right\Vert_L\leq 1}\left\vert\int_{\Ss} p(s_2\vert s_1,a_1) \int_{\Ss} f(s')\left( b(s' \vert s_2,a_2,\cdots,a_\delay) -\delta_{s_1}(s')\right)ds' \right\vert
        \label{eq:cond_s2}
        \\
        &= \sup_{\left\Vert f\right\Vert_L\leq 1}\left\vert\int_{\Ss} p(s_2\vert s_1,a_1) \int_{\Ss} f(s')\left( b(s' \vert s_2,a_2,\cdots,a_\delay) -\delta_{s_2}(s')\right)ds' \right.\nonumber
        \\
        &\quad + \left.\int_{\Ss} p(s_2\vert s_1,a_1) \int_{\Ss} f(s')\left( \delta_{s_2}(s) -\delta_{s_1}(s')\right)ds' \right\vert
        \label{eq:add_remove_delta}
        \\
        &\leq \underbrace{\sup_{\left\Vert f\right\Vert_L\leq 1}\left\vert\int_{\Ss} p(s_2\vert s_1,a_1) \int_{\Ss} f(s')\left( b(s' \vert s_2,a_2,\cdots,a_\delay) -\delta_{s_2}(s')\right)ds'  \right\vert}_A\nonumber
        \\
        &\quad + \underbrace{\wass\left(P(\vert s_1,a_1)\Vert \delta_{s_1}\right)}_B
        \nonumber,
    \end{align}
    where \eqref{eq:cond_s2} hols by conditioning on the second visited state $s_2$ and \Cref{eq:add_remove_delta} holds by adding and subtracting $\delta_{s_2}$.
    The reader may have recognized that the statement at $\delay-1$ can be used to bound $A$ while $B$ can be bounded with the $L_T$-TLC assumption. Therefore 
    \begin{align*}
        \wass\left(b(\cdot\vert x)\Vert \delta_{s_1} \right) \leq dL_T,
    \end{align*}
    and the statement holds for any $\delay\in\mathbb{N}$.
\end{proof}

\begin{lem}[Time-Lipschitz bound]
\label{lem:bound_sigma_tlc}
Consider an $L_T$-Lipschitz MDP with delay $\delay$. Then, one has
    \begin{align*}
        \sigma_b^\rho  \leq 2\delay L_T
    \end{align*}
\end{lem}
\begin{proof}
    Call $s_x$ the state contained in $x$. By triangular inequality, one has
    \begin{align}
        \sigma_b^\rho
        &= \E_{\substack{x\sim d^{\pidel}_\rho\\ s,s'\sim b(\cdot\vert x)}}\left[ \dist_{\Ss}(s,s') \right]\nonumber
        \\
        &\leq \E_{\substack{x\sim d^{\pidel}_\rho\\ s\sim b(\cdot\vert x')}}\left[ \dist_{\Ss}(s,s_x) \right] + \E_{\substack{x\sim d^{\pidel}_\rho\\ s'\sim b(\cdot\vert x)}}\left[ \dist_{\Ss}(s_x,s') \right]\nonumber
        \\
        &= 2\E_{x\sim d^{\pidel}_\rho}\left[\int_{\Ss} d_{\Ss}(s,s_x) b(s\vert x)~ds\right] \nonumber
        \\
        &= 2\E_{x\sim d^{\pidel}_\rho}\left[\int_{\Ss} \dist_{\Ss}(s,s_x) \left(b(s\vert x) - \delta_{s_x}(s) \right)~ds\right]\label{eq:null_term}
        \\
        &\leq 2 \E_{x\sim d^{\pidel}_\rho}\left[\wass\left(b(\cdot\vert x)\Vert \delta_{s_x} \right)\right]\label{eq:wass_recognise},
    \end{align}
    where \Cref{eq:null_term} holds because $\int_{\Ss}d_{\Ss}(s,s_x)\delta_{s_x}(s)ds=0$ and \Cref{eq:wass_recognise} follows by recognizing the Wasserstein distance.
    One can then use \Cref{pp:tlc_delay} on each of the two terms to conclude.
\end{proof}

\section{Bounding the Value Function via Performance Difference Lemma}
\label{app:proofs}

\perfdifflem*
\begin{proof}
We first prove the result for integer delay $d\in\mathbb N$.
    We start by adding and subtracting $\E_{\substack{s\sim b(\cdot\vert x')\\a\sim \pidel(\cdot\vert x')}}\left[r(s,a) + \gamma \E_{s'\sim p(\cdot\vert s,a)}[V^{\pi_E}(s')]\right]$ to the quantity of interest $I(x)=\E_{s\sim b(\cdot\vert x)}[V^{\pi_E}(s)] - V^{\pidel}(x)$.
    
    This yields:
    \begin{align*}
        I(x)
        &= \underbrace{\E_{s\sim b(\cdot\vert x)}[V^{\pi_E}(s)] - \E_{\substack{s\sim b(\cdot\vert x')\\a\sim \pidel(\cdot\vert x')}}\left[r(s,a) + \gamma \E_{s'\sim p(\cdot\vert s,a)}[V^{\pi_E}(s')]\right]}_A
        \\
        &\quad + \underbrace{\E_{\substack{s\sim b(\cdot\vert x')\\a\sim \pidel(\cdot\vert x')}}\left[r(s,a) + \gamma \E_{s'\sim p(\cdot\vert s,a)}[V^{\pi_E}(s')]\right] - V^{\pidel}(x)}_B.
    \end{align*}
    
    The first term is 
    \begin{align*}
        A = \E_{s\sim b(\cdot\vert x)}[V^{\pi_E}(s)] - \E_{\substack{s\sim b(\cdot\vert x')\\a\sim \pidel(\cdot\vert x')}}\left[Q^{\pi_E}(s,a)\right].
    \end{align*}
    For the second term, note that $V^{\pidel}(x) = \E_{\substack{s\sim b(\cdot\vert x')\\a\sim \pidel(\cdot\vert x')}}\left[r(s,a)\right]+ \gamma \E_{\substack{x'\sim \pdel(\cdot\vert x,a)\\a\sim \pidel(\cdot\vert x')}}[V^{\pidel}(x')]$. Therefore,
    \begin{align*}
        B &= \gamma\E_{\substack{s\sim b(\cdot\vert x')\\a\sim \pidel(\cdot\vert x')}}\left[ \E_{s'\sim p(\cdot\vert s,a)}[V^{\pi_E}(s')]\right] - \gamma \E_{\substack{x'\sim \pdel(\cdot\vert x,a)\\a\sim \pidel(\cdot\vert x')}}[V^{\pidel}(x')].
    \end{align*}
    By observing that $\int_{\Ss} b(s\vert x)p(s'\vert s,a)ds = \int_{\Xs}\tilde p(x'\vert x,a)b(s'\vert x')dx'$, that is, compute the current state with the belief then the next current state is equivalent to computing the next extended state and then the next current state with the belief. Thus,
    \begin{align*}
        B &= \gamma \E_{\substack{x'\sim \pdel(\cdot\vert x,a)\\a\sim \pidel(\cdot\vert x')}}\left[\E_{s\sim b(\cdot\vert x')}[V^{\pi_E}(s')] - V^{\pidel}(x') \right]
        \\
        &= \gamma \E_{\substack{x'\sim \pdel(\cdot\vert x,a)\\a\sim \pidel(\cdot\vert x')}}\left[I(x')\right],
    \end{align*}
    where we have recognised the quantity of interest $I$ taken at another extended state. One can thus iterate as in the original performance difference lemma to get
    \begin{align}
        I(x) &= \sum_{t=0}^\infty \gamma^t \mathbb{E}^{\pidel}\left[\left. \E_{s\sim b(\cdot\vert x_t)}[V^{\pi_E}(s)] - \E_{\substack{s\sim b(\cdot\vert x_t)\\a\sim \pidel(\cdot\vert x_t)}}[Q^{\pi_E}(s,a)] \right\vert x_0=x \right]\nonumber
        \\
        &= \frac{1}{1-\gamma}\E_{x'\sim d^{\pidel}_{x}}\left[ \E_{s\sim b(\cdot\vert x')}[V^{\pi_E}(s)] - \E_{\substack{s\sim b(\cdot\vert x')\\a\sim \pidel(\cdot\vert x')}}[Q^{\pi_E}(s,a)] \right],\label{eq:state_distrib_reco}
    \end{align}
    where \Cref{eq:state_distrib_reco} is obtained by recognising the discounted state distribution under policy $\pidel$. This completes the proof.
    
    We now assume non-integer delay, setting $\delay\in(0,1)$ but the proof for $\delay\in\mathbb R$ follows easily. 
    The proof is the same as above except for substituting $b$ with $b_\delay$. One step which might not be evident is that
    \begin{align*}
        \int_{\Ss} b_\delay(s_{t+\delay}\vert x_t)p(s_{t+1+\delay}\vert s_{t+\delay},a)~ds_{t+\delay} = \int_{\Xs} \tilde p(x_{t+1}\vert x_t,a)b_\delay(s_{t+1+\delay}\vert x_{t+1})~dx_{t+1}.
    \end{align*}
    This is true because, for $x_t=(s_t,a_{t-1}), x_{t+1}=(s_{t+1},a_t)\in\Xs$,
    \begin{align}
    \int_{\Ss} b_\delay(s_{t+\delay}\vert x_t)&p(s_{t+1+\delay}| s_{t+\delay},a)~ds_{t+\delay}\nonumber
    \\
    &=\int_{\Ss} b_\delay(s_{t+\delay}\vert x_t) \int_{\Ss}  b_\delay(s_{t+1+\delay}\vert s_{t+1},a) b_{1-\delay}(s_{t+1}\vert s_{t+\delay},a)~ds_{t+1}~ds_{t+\delay} \label{eq:use_def_p_non_int}
    \\
    &= \int_{\Ss} b_\delay(s_{t+1+\delay}\vert s_{t+1},a)\int_{\Ss} b_\delay(s_{t+\delay}\vert x_t)b_{1-\delay}(s_{t+1}\vert s_{t+\delay},a)~ds_{t+\delay}~ds_{t+1} \nonumber
    \\
    &= \int_{\Ss} b_\delay(s_{t+1+\delay}\vert s_{t+1},a)\int_{\Ss} b_\delay(s_{t+\delay}\vert s_t,a_{t-1})b_{1-\delay}(s_{t+1}\vert s_{t+\delay},a)~ds_{t+\delay}~ds_{t+1} \nonumber
    \\
    &= \int_{\As} \int_{\Ss} b_\delay(s_{t+1+\delay}|s_{t+1},a_{t}) \delta_{a}(a_{t})\int_{\Ss} b_\delay(s_{t+\delay}\vert s_t,a_{t-1}) b_{1-\delay}(s_{t+1}\vert s_{t+\delay},a)~ds_{t+\delay}~ds_{t+1}~da_{t}\label{eq:def_aug_mdp_p_non_int}
    \\
    &= \int_\Xs b_\delay(s_{t+1+\delay}\vert x_{t+1})\tilde p(x_{t+1}\vert x_t, a)~dx_{t+1}\nonumber
    \end{align}
    where \Cref{eq:use_def_p_non_int} holds by replacing the transition $p$ as in \Cref{eq:def_p_non_int} and \Cref{eq:def_aug_mdp_p_non_int} holds by definition of the transition in the augmented MDP.
\end{proof}

\perfdiffbound*
\begin{proof}
    We first prove the result for integer delay $\delay\in\mathbb N$.
    The first step in this proof is to use the results of \Cref{lem:perfdimlem}, which yields, for any $x \in \Xs$:
    \begin{align*}
        \E_{s\sim b(\cdot\vert x)}&[V^{\pi_E}(s)] - V^{\pidida}(x) 
        \leq \frac{1}{1-\gamma} \E_{x'\sim d^{\pidida}_x} \left[\underbrace{ \E_{s\sim b(\cdot\vert x')}[V^{\pi_E}(s)] -  \E_{\substack{s\sim b(\cdot\vert x')\\a\sim \pidida(\cdot\vert x')}}[Q^{\pi_E}(s,a)]}_A \right].
    \end{align*}
    We consider the term inside the expectation over $x'$, called $A$. We reformulate this term to highlight how we then apply \Cref{pp:exp_q_bound}.
    \begin{align}
        A 
        &= \E_{s\sim b(\cdot\vert x')}\left[\E_{\substack{a_1\sim\pi_E(\cdot\vert s) \\a_2\sim \pidida(\cdot\vert x')}}\left[ Q^{\pi_E}(s,a_1) -  Q^{\pi_E}(s,a_2) \right]\right]\nonumber
        \\
        &\leq L_Q \E_{s\sim b(\cdot\vert x')} \left[ \wass(\pi_E(\cdot\vert s)
        \Vert \pidida(\cdot\vert x') ) \right].\label{eq:partial_res_v}
    \end{align}
    To finish the proof, it remains to bound below $\sigma_b^x\coloneq \E_{\substack{x'\sim d^{\pidida}_x\\ s,s'\sim b(\cdot\vert x')}}\left[ \dist_{\Ss}(s,s') \right]$ 
    with
    $\E_{s\sim b(\cdot\vert x')} \left[ \wass(\pi_E(\cdot\vert s) \Vert \pidida(\cdot\vert x') ) \right]$.
    
    One has
    \begin{align}
         \wass(\pi_E(\cdot\vert s) \Vert \pidida(\cdot\vert x') ) 
        &= \sup_{\left\Vert f\right\Vert_L\leq 1}\left\vert\int_{\As} f(a)(\pi_E(a\vert s)-\pidida(a\vert x')) da\right\vert \nonumber
        \\
        &= \sup_{\left\Vert f\right\Vert_L\leq 1}\left\vert\int_{\As} f(a)(\pi_E(a\vert s)-\int_{\Ss}\pi_E(a\vert s')b(s'\vert x')ds') da\right\vert \label{eq:def_pidel}
        \\
        &\leq  \int_{\Ss} b(s'\vert x') \sup_{\left\Vert f\right\Vert_L\leq 1}\left\vert\int_{\As} f(a)(\pi_E(a\vert s)-\pi_E(a\vert s')) da \right\vert ds' \label{eq:fub_ton_sigma}
        \\
        &\leq \int_{\Ss} b(s'\vert x') \wass(\pi_E(\cdot\vert s) \Vert \pi_E(\cdot\vert s') ) ds'\nonumber
        \\
        &\leq L_\pi \int_{\Ss} b(s'\vert x')  \dist_{\Ss}(s,s') ds'
        \label{eq:lip_expert_pol}
        ,
    \end{align}
    where \Cref{eq:def_pidel} holds by definition of the optimal imitated delayed policy (see \Cref{eq:belief_pol}), \Cref{eq:fub_ton_sigma} holds by application of Fubini-Tonelli's theorem and \Cref{eq:lip_expert_pol} holds by Lipschitzness of the expert undelayed policy.
    By re-injecting this result into \Cref{eq:partial_res_v} we get the desired result.
    
    We now assume non-integer delay, setting $\delay\in(0,1)$ but the proof for $\delay\in\mathbb R$ follows easily. 
    In this case, the optimal policy learnt by DIDA is
    \begin{align}
        \pidida(a\vert x) = \int_{\Ss} b_\delay(s\vert x) \pi_E (a\vert s) ds.\label{eq:belief_pol_non_int}.
    \end{align}
    The proof remains the same except for the first step, where the performance difference lemma is of course replaced by its non-integer delay version just discussed.
\end{proof}

\perfdiffboundeucl*
\begin{proof}
    The result follows from application of \Cref{lem:bound_sigma_eucl} to \Cref{th:perf_diff_bound}.
\end{proof}

\perfdiffboundtlc*
\begin{proof}
    The result follows from application of \Cref{lem:bound_sigma_tlc} to \Cref{th:perf_diff_bound}.
\end{proof}

\subsection{Lower Bounding the Value Function}
\label{app:tightness}
\lbtight*
\begin{proof}
    We consider an MDP $\mathcal{M} = (\Ss, \As, p, R, \mu, \gamma)$ such that $\Ss=\mathbb R$ and $\As = \mathbb R$, its state transition is given by $s_{t+1}=s_t+\frac{a}{L_\pi}+\varepsilon_t$, where $\varepsilon_t \overset{i.i.d.}\sim \mathcal N (0,\sigma^2)$. The transition distribution can be written as 
    \begin{align*}
        p(s'|s,a)=\mathcal N\left(s';\  s+\frac{a}{L_\pi},\sigma^2\right).
    \end{align*}
    Defining $L_r:=L_Q L_\pi$, the reward is given by $r(s,a)=-L_r\left\vert s+\frac{a}{L_\pi}\right\vert$.
    Note that the reward is always negative, yet the policy $\pi^*(\cdot|s)=\delta_{-L_\pi s}(s')$ always yields $0$ reward and is therefore optimal. Clearly, its value function satisfies $V^*(s)=0$ for every $s\in \Ss$. 
    Its $Q$ function is 
    \begin{align*}
        Q^*(s,a)&=-L_r\left\vert s+\frac{a}{L_\pi}\right\vert + \gamma \int_{\mathbb R} V^*(s')\ p(s'\vert s,a)\ ds'
        \\
        &= -L_r\left\vert s+\frac{a}{L_\pi}\right\vert
        \\
        &=-L_{Q}\left\vert L_\pi s+a\right\vert.
    \end{align*}
    Therefore, $Q^*$ is indeed $L_Q$-LC in the second argument.
    
    Consider now any $\delay$-delayed policy $\pidel$.
    At each time step $t$, the current state $s_t$ can be decomposed in this way
    \begin{align*}
        s_t=\underbrace{s_{t-\delay-1}+\sum_{\tau=t-\delay-1}^{t-1} \frac{a_\tau}{L_\pi}}_{\eqqcolon \phi(x_t)}+\underbrace{\sum_{\tau=t-\delay-1}^{t-1} \varepsilon_\tau}_{\epsilon},
    \end{align*}
    where the first quantity is a deterministic function of the extended state, while the second is distributed under $\mathcal N(0,d\sigma^2)$. 
    The expected value of the instantaneous reward is then given by
    \begin{align*}
        \E[r(x_t,a)]
        &= -L_r\E\left[\left\vert s_t+\frac{a}{L_\pi}\right\vert\right]
        \\
        &=-L_r\E\left[\left\vert\phi(x_t)+\frac{a}{L_\pi}+\sum_{\tau=t-\delay-1}^{t-1} \varepsilon_\tau\right\vert\right].
    \end{align*}
    The function $f:y\mapsto \E\left[\left\vert \mathcal N(y,\sigma)\right\vert\right]$ has minimum at 0  by symmetry of the normal distribution. Its value is the mean of a half-normal distribution, that is  $\E\left[\left\vert \mathcal N(0,\sigma)\right\vert\right]=\frac{\sqrt 2}{\sqrt \pi}\sigma$.
    Therefore
    \begin{align*}
        \E[r(x_t,a)]
        &\leq -L_r\frac{\sqrt 2}{\sqrt \pi}\sqrt{d}\sigma,
    \end{align*}
    which implies
    \begin{align*}
        V^{\pidel}(x_t)
        &\leq-\frac{L_r}{1-\gamma}\frac{\sqrt 2}{\sqrt \pi}\sqrt{d}\sigma
        \\
        &= -\frac{L_Q L_\pi}{1-\gamma}\frac{\sqrt 2}{\sqrt \pi}\sqrt{ \Var_{s\sim b(\cdot|x')}(s|x')}
        ,
    \end{align*}
    by noticing that $L_r=L_Q L_\pi$ and that
    $\mathbb Var_{s\sim b(\cdot|x)}(s|x)=d\sigma^2$. Note that $\sqrt{ \Var_{s\sim b(\cdot|x)}(s|x)}$ is the same for each $x\in\Xs$ so we can replace it with $\E_{x'\sim d_x^{\tilde \pi}(\cdot)}\left[\sqrt{ \Var_{s\sim b(\cdot|x')}(s|x')}\right]$ to have a result more similar to \Cref{th:perf_diff_bound}.
    
    Recalling that the optimal value function had 
    value 0 at any state concludes the proof.
\end{proof}

\section{Bounding the Value Function via the State Distribution}
\label{app:bound_state_dist}
For this bound, we wish to use the difference in state distribution between the delayed and the undelayed expert to grasp their difference. 
Obviously, they do not share the same state space since the DMDP is handled by augmenting the state. 
However, it is possible to define a unifying framework with the following object.
In this section, we consider $\delay\in\mathbb N$.

\begin{defi}[$\delay^{\text{th}}$-order MDP]
    Given an MDP $\mathcal{M} = (\Ss, \As, p, R, \mu)$, we define its correspondent $\delay^{\text{th}}$-order MDP, $\delay\in\mathbb N$, as the MDP $\widebar{\mathcal{M}}$ (a \say{$\ \widebar{ }\ $} will be used to refer to an element of an $\delay^{\text{th}}$-order MDP) with
    \begin{itemize}
        \item State space $\widebar{\Ss} = \Ss^{\delay+1}\times \As^\delay$, whose states $\widebar s$ are composed of the last $\delay+1$ states and $\delay$ actions of the MDP, namely $\widebar{s}= (s_{1},a_{1},s_{2},a_{2},\dots,s_{\delay},a_{\delay},s_{\delay+1})$.
        \item Unchanged action space $\As$.
        \item Reward function $R$, overwriting the undelayed notation but using as input a $\delay^{\text{th}}$-order state, such that\\ $R(\widebar{s}_t,a)=R(s_{t},a)$. The overwriting is justified by this equality.
        \item Transition function $\widebar{p}$ given by
        \begin{align}
            \widebar{p}(\widebar{s}'\vert \widebar s,a)=p(s_{\delay+1}'\vert s_{\delay+1},a)\delta_a(a_\delay')\prod_{i=1}^\delay \delta_{s_{i+1}}(s_i')\prod_{i=1}^{d-1}\delta_{a_{i+1}}(a_i'),\label{eq:p_order_d}
        \end{align}
        where
        $\widebar s = (s_1,a_1,\dots,a_\delay,s_{\delay+1})$ and $\widebar s' = (s_1',a_1',\dots,a_\delay',s_{\delay+1}')$ .
        \item The initial state distribution $\widebar \mu$ is such that the inital action queue $x$ is distributed as in the delayed MDP while the states of the states queue are distributed as
        $s_{i+1}\sim p(\cdot|s_i,a_i)$.
    \end{itemize}
\end{defi}

This definition is inspired from the concept of $\delay^{\text{th}}$-order Markov chain.
In this definition, $s_{\delay+1}$ is intended to be the current state, while $s_1$ is the $\delay$-delayed state, $a_{1:\delay}$ is the action queue.
Therefore, from the state of the $\delay^{\text{th}}$-order MDP, one can either extract an extended state and query a delayed policy or the extract the current state and query an undelayed policy.
This implies that one can define the state distribution on the $\delay^{\text{th}}$-order MDP for both an undelayed and a delayed policy. We overwrite the notations and write respectively  $d^\pi$ and $d^{\widebar \pi}$ the distribution of state on the $\delay^{\text{th}}$-order MDP. 
The fact that the distribution concerns a $\delay^{\text{th}}$-order state and not a state from the underlying MDP or DMDP will be clear from the notation of the variable which is sampled from this distribution. For instance, in $s\sim d^\pi$, $d^\pi$ is a distribution defined by applying $\pi$ on the undelayed MDP while $\tilde s\sim d^\pi$ assumes a distribution under $\pi$ on the $\delay^{\text{th}}$-order MDP.

Before deriving bounds on the $\delay^{\text{th}}$-order state probability distribution, we first prove a Lipschitzness result concerning the $\delay^{\text{th}}$-order MDP which be used in later proofs.

\begin{lem}
\label{lem:g_f_lip}
    Consider an $(L_P,L_r)$-LC MDP and its $\delay^{\text{th}}$-order MDP counterpart. Let $f:\widebar{\Ss}\to\mathbb{R}$ such that $\lVert f\rVert_L \leq 1$ w.r.t. to the L2-norm on $\widebar{\Ss}$.
    Then, the function 
    \begin{align*}
        g_f :\widebar{\Ss}\times\As&\longrightarrow\mathbb{R}\\
        (\widebar s,a)&\longmapsto \int_{\widebar{\Ss}} f(\widebar{s}')\widebar{p}(\widebar s'\vert \widebar s,a)~d\widebar{s}',
    \end{align*}
    is $L_P$-LC w.r.t. the second variable.
\end{lem}
\begin{proof}
    Let $\widebar s\in \widebar{\Ss}$ such that $\widebar s = (s_1,a_1,\dots,a_\delay,s_{\delay+1})$ and $a,b\in\As$. Then,
    \begin{align}
        \left\vert g_f(\widebar s,a)-g_f(\widebar{s},b)\right\vert 
        &= \left\vert  \int_{\widebar{\Ss}} f(\widebar{s}') \left(\widebar{p}(\widebar s'\vert \widebar s,a)-\widebar{p}(\widebar s'\vert \widebar s,b)\right)~d\widebar{s}'\right\vert \nonumber
        \\
        &= \left\vert  \int_{\Ss\times\As} f(s_2,a_2,\dots,s_{\delay+1},a_{\delay+1}',s_{\delay+1}') \right.
        \\
        &\quad\left.\vphantom{\int}\left({p}(s_{\delay+1}'\vert s_{\delay+1},a)\delta_a(a_{\delay+1}')-{p}(s_{\delay+1}'\vert s_{\delay+1},b)\delta_b(a_{\delay+1}')\right)~ds_{\delay+1}'~da_{\delay+1}'\right\vert \label{eq:int_dirac}
    \end{align}
    where in \Cref{eq:int_dirac} we integrate over the elements of $\widebar s$ fixed by the Dirac distributions of \Cref{eq:p_order_d} except for the last action. 
    Note that $h\coloneqq (s,a)\mapsto f(s_2,a_2,\dots,s_{\delay+1},a,s)$ is 1-LC because $f$ is 1-LC.
    We add and substract the quantity ${p}(s_{\delay+1}'\vert s_{\delay+1},b)\delta_a(a_{\delay+1})$ inside the integral to get
    \begin{align*}
        \left\vert g_f(\widebar s,a)-g_f(\widebar{s},b)\right\vert 
        &\leq \left\vert  \int_{\Ss\times\As}\delta_a(a_{\delay+1}') h(a_{\delay+1}',s_{\delay+1}') \left({p}(s_{\delay+1}'\vert s_{\delay+1},a)-{p}(s_{\delay+1}'\vert s_{\delay+1},b)\right)~ds_{\delay+1}'~da_{\delay+1}'\right\vert
        \\
        &\quad  + \left\vert  \int_{\Ss\times\As} h(a_{\delay+1}',s_{\delay+1}'){p}(s_{\delay+1}'\vert s_{\delay+1},b) \left(\delta_a(a_{\delay+1}')-\delta_b(a_{\delay+1}')\right)~ds_{\delay+1}'~da_{\delay+1}'\right\vert
        \\
        &\leq (L_P+1) \dist_{\As}(a,b),
    \end{align*}
    where the last inequality follows by integrating over $s_{\delay+1}'$ for the first integral and $a_{\delay+1}'$ for the second, before taking the supremum over functions $h$ and recognising the Wasserstein distance. 
\end{proof}

We can now prove a first important intermediary result which bounds the Wasserstein divergence in discounted state distribution in the $\delay^{\text{th}}$-order MDP between the undelayed and the delayed policy.

\begin{thm}
\label{th:state_distrib_bound}
    Consider an $(L_P,L_r)$-LC MDP $\mathcal{M}$ and its $(L_{\widebar{P}},L_r)$-LC $\delay^{\text{th}}$-order MDP counterpart $\widebar{\mathcal{M}}$.  Let $\pi_E$ be a $L_\pi$-LC undelayed policy and assume that $\gamma L_P(1+L_\pi)\le 1$. 
    Let $\pidida$ be a delayed policy as defined in \Cref{eq:belief_pol}. 
    Then, the two discounted state distributions $d_{\widebar{\mu}}^{\pi_E}, d_{\widebar{\mu}}^{\pidida}$ defined on $\widebar{\mathcal{M}}$ satisfy
    \begin{align*}
        \wass \left(d_{\widebar{\mu}}^{\pi_E} \Vert d_{\widebar{\mu}}^{\pidida}\right) \leq \gamma L_Q (1+L_{P})
        \sigma_b^{\mudel}
    \end{align*}
    where $\sigma_b^{\mudel} = \E_{\substack{x\sim d^{\pidida}_{\mudel}\\ s,s'\sim b(\cdot\vert x)}}\left[ \dist_{\Ss}(s,s') \right]$ and $d^{\pidida}_{\mudel}$ is defined on the DMDP.
\end{thm}
\begin{proof}
    We start by developing the term $\wass \left(d_{\widebar{\mu}}^{\pi_E}\Vert d_{\widebar{\mu}}^{\pidida}\right)$, using the supremum over the space of functions $f:\widebar{\Ss}\to\mathbb{R}$ such that $\lVert f\rVert_L \leq 1$ w.r.t. to the L2-norm on $\widebar{\Ss}$. 
    \begin{align*}
        \wass \left(d_{\widebar{\mu}}^{\pi_E} \Vert d_{\widebar{\mu}}^{\pidida}\right) = \sup_{\left\Vert f\right\Vert_L\leq 1}\left\vert\int_{\widebar{\Ss}} f(\widebar{s})\left(d_{\widebar{\mu}}^{\pi_E}(\widebar{s})- d_{\widebar{\mu}}^{\pidida}(\widebar{s})\right)~d\widebar{s} \right\vert.
    \end{align*}
    We then use the fact that for some policy $\pi\in\Pi$, $d_{\widebar{\mu}}^{\pi}(\widebar{s}) = (1-\gamma)\widebar{\mu}(\widebar{s}) + \gamma\int_{\widebar{\Ss}}\widebar{p}^{\pi}(\widebar s\vert \widebar s')d_{\widebar{\mu}}^{\pi}(\widebar{s}')~d\widebar s'$, where $\widebar{p}^{\pi}(\widebar s\vert \widebar s')=\int_{\As}p(\widebar s'\vert \widebar s,a)\pi(a\vert \widebar{s})~da$ to yield 
    \begin{align}
        \wass \left(d_{\widebar{\mu}}^{\pi_E} \Vert d_{\widebar{\mu}}^{\pidida}\right) 
        &= \gamma \sup_{\left\Vert f\right\Vert_L\leq 1}\left\vert\int_{\widebar{\Ss}} f(\widebar{s})\int_{\widebar{\Ss}}\left(\widebar{p}^{\pi_E}(\widebar s\vert \widebar s')d_{\widebar{\mu}}^{\pi_E}(\widebar{s}')-\widebar{p}^{\pidida}(\widebar s\vert \widebar s') d_{\widebar{\mu}}^{\pidida}(\widebar{s}')\right)~d\widebar{s}'~d\widebar{s} \right\vert \nonumber
        \\
        &\leq \gamma \underbrace{\sup_{\left\Vert f\right\Vert_L\leq 1}\left\vert\int_{\widebar{\Ss}} f(\widebar{s})\int_{\widebar{\Ss}}\left(d_{\widebar{\mu}}^{\pi_E}(\widebar{s}')-d_{\widebar{\mu}}^{\pidida}(\widebar{s}')\right)\widebar{p}^{\pi_E}(\widebar s\vert \widebar s')~d\widebar{s}'~d\widebar{s} \right\vert}_A \nonumber
        \\
        &\quad + \gamma \underbrace{\sup_{\left\Vert f\right\Vert_L\leq 1}\left\vert\int_{\widebar{\Ss}} f(\widebar{s})\int_{\widebar{\Ss}}\left(\widebar{p}^{\pi_E}(\widebar s\vert \widebar s')-\widebar{p}^{\pidida}(\widebar s\vert \widebar s') \right)d_{\widebar{\mu}}^{\pidida}(\widebar{s}')~d\widebar{s}'~d\widebar{s} \right\vert}_B \label{eq:introduce_term},
    \end{align}
    where \Cref{eq:introduce_term} follows by adding and subtracting the term $\widebar{p}^{\pi_E}(\widebar s\vert \widebar s')d_{\widebar{\mu}}^{\pidida}(\widebar{s}')$ and using the triangular inequality.
    
    The first term $A$ can be bounded by using Fubini's theorem first and then leveraging \Cref{pp:lip_g} which implies that $\int_{\widebar{\Ss}} f(\widebar{s}) \widebar{p}^{\pi_E}(\widebar s\vert \widebar s')~d\widebar{s}'$ is $L_P(1+L_\pi)$-LC.
    Therefore
    \begin{align*}
        A \leq L_P(1+L_\pi) \wass \left(d_{\widebar{\mu}}^{\pi_E} \Vert d_{\widebar{\mu}}^{\pidida}\right) 
    \end{align*}
    
    Now looking at the second term $B$, we will develop the term $\widebar{p}^{\pi_E}$ and $\widebar{p}^{\pidida}$ to highlight the influence of the policy and leverage \Cref{eq:belief_pol}. We note $\widebar s = (s_1,a_1,\dots,a_\delay,s_{\delay+1})$ and $\widebar s' = (s_1',a_1',\dots,a_\delay',s_{\delay+1}')$. Moreover, we overwrite the notation of the belief to use it on a $\delay^{\text{th}}$-order state such that $b(z\vert \widebar s) = b(z\vert s_1,a_1,\dots,a_\delay)$, that is, the belief is based on the augmented state constructed from the oldest state inside $\widebar s$ and the sequence of action it contains.
    \begin{align*}
        B &= \sup_{\left\Vert f\right\Vert_L\leq 1}\left\vert\int_{\widebar{\Ss}} f(\widebar{s})\int_{\widebar{\Ss}}\left(\int_{\As} \widebar{p}(\widebar s\vert \widebar s',a)\pi_E(a \vert  s_{\delay+1}')~da -\int_{\As}\int_{z\in\Ss}\widebar{p}(\widebar s\vert \widebar s',a) b(z\vert \widebar{s}')\pi_E(a\vert z) ~dz~da\right)d_{\widebar{\mu}}^{\pidida}(\widebar{s}')~d\widebar{s}'~d\widebar{s} \right\vert
        \\
        &= \sup_{\left\Vert f\right\Vert_L\leq 1}\left\vert\int_{\widebar{\Ss}} f(\widebar{s})\left(\int_{\widebar{\Ss}}\int_{\As} \underbrace{\left(\pi_E(a \vert  s_{\delay+1}') -\int_{z\in\Ss} b(z\vert \widebar{s}')\pi_E(a\vert z) ~dz\right)}_{C}\widebar{p}(\widebar s\vert \widebar s',a)~da~d_{\widebar{\mu}}^{\pidida}(\widebar{s}')~d\widebar{s}'~d\widebar{s} \right)\right\vert,
    \end{align*}
    where the term $C$ in this equation accounts for the difference in taking an action with the undelayed policy $\pi_E$ instead of the belief-based policy. 
    Fubini's theorem yields 
    \begin{align*}
        B &=
        \sup_{\left\Vert f\right\Vert_L\leq 1}\left\vert \int_{\widebar{\Ss}}\int_{\As} \left(\pi_E(a \vert  s_{\delay+1}') -\int_{z\in\Ss} b(z\vert \widebar{s}')\pi_E(a\vert z) ~dz\right)\underbrace{\left(\int_{\widebar{\Ss}} f(\widebar{s})\widebar{p}(\widebar s\vert \widebar s',a)~d\widebar{s}\right)}_{g_f(\widebar s',a)} ~da~d_{\widebar{\mu}}^{\pidida}(\widebar{s}')~d\widebar{s}' \right\vert,
    \end{align*}
    where we note $g_f(\widebar s',a)\coloneqq \int_{\widebar{\Ss}} f(\widebar{s})\widebar{p}(\widebar s\vert \widebar s',a)~d\widebar{s}$. This function is $(1+L_{P})$-LC in $a$ by \Cref{lem:g_f_lip}. 
    Noting also that $\pi_E(a \vert \widebar s_{\delay+1}')=\int_{z\in\Ss}\pi_E(a \vert \widebar s_{\delay+1}')b(z\vert \widebar{s}')dz$, one gets
    \begin{align*}
        B &=
        \sup_{\left\Vert f\right\Vert_L\leq 1}\left\vert \int_{\widebar{\Ss}} \int_{z\in\Ss} \int_{\As} \left(\pi_E(a \vert s_{\delay+1}') - \pi_E(a\vert z) \right)g_f(\widebar s',a)~da b(z\vert \widebar{s}')~dz ~d_{\widebar{\mu}}^{\pidida}(\widebar{s}')~d\widebar{s}' \right\vert
    \end{align*}
    
    Then, by Lipschitzness of $\pi_E$,
    \begin{align*}
        B &\leq L_\pi (1+L_{ P}) \int_{\widebar{\Ss}} \int_{z\in\Ss} \dist_{\Ss}(s
        _{\delay+1}',z) b(z\vert \widebar{s}')~dz ~d_{\widebar{\mu}}^{\pidida}(\widebar{s}')~d\widebar{s}' .
    \end{align*}
    
    Now, one can observe that  $s\mapsto\int_{\widebar{\Ss}}\delta(s=s_{\delay+1})d_{\widebar{\mu}}^{\pidida}(\widebar{s})~d\widebar s$ and  $s\mapsto\int_{\Xs}b(s\vert x)d_{\mudel}^{\pidida}(x)~dx$ define the same distribution over $\Ss$. Note that the discounted state distributions $d_{\widebar{\mu}}^{\pidida}(\widebar{s})$ is over the $\delay^{\text{th}}$-order MDP's state space while $d_{\mudel}^{\pidida}(x)$ is over the DMDP's state space.
    This yields
    \begin{align*}
        B&\leq L_\pi (1+L_{ P}) \int_{\Xs} \int_{s'\in\Ss}\int_{s\in\Ss} d_{\Ss}(s,s') b(s\vert x)b(s'\vert x)~ds~ds' ~d_{\mudel}^{\pidida}(x)~dx
        \\
        &= L_\pi (1+L_{ P})  \E_{\substack{x\sim d^{\pidida}_{\mudel}\\ s,s'\sim b(\cdot\vert x)}}\left[ \dist_{\Ss}(s,s') \right]
        \\
        &= L_\pi (1+L_{ P}) \sigma_b^{\mudel}.
    \end{align*}
    
    One can now resume at \Cref{eq:introduce_term}, and recording that $\gamma L_P(1+L_\pi)\le 1$, one gets
    \begin{align*}
        \wass \left(d_{\widebar{\mu}}^{\pi_E} \Vert d_{\widebar{\mu}}^{\pidida}\right) 
        &\leq \gamma L_P(1+L_\pi) \wass \left(d_{\widebar{\mu}}^{\pi_E} \Vert d_{\widebar{\mu}}^{\pidida}\right) + \gamma L_\pi (1+L_{ P}) \sigma_b^{\mu}
        \\
        &\leq \frac{\gamma L_\pi (1+L_{ P})}{1-\gamma L_P(1+L_\pi)}\sigma_b^{\mudel}.
    \end{align*}
    Finally, recalling that $L_Q=\frac{L_r}{1-\gamma L_P(1+L_\pi)}$ finishes the proof.
\end{proof}

the previous results will be useful to prove a bound on the value function. However, before providing this result, we need a last intermediary result. 

\begin{prop}
\label{pp:order_d_reward}
The $\delay^{\text{th}}$-order MDP $\widebar{\mathcal M}$ can be reduced to an $\delay^{\text{th}}$-order MDP where the mean reward is redefined as $\widebar{r}(\widebar s,a) = r(s_1,a_1)$. Because it doesn't depend upon $a$, we equivalently write  $\widebar{r}(\widebar s) \coloneqq \widebar{r}(\widebar s,a)$
\end{prop}
\begin{proof}
    We derive a similar proof as \citet{katsikopoulos2003markov}.
    Let $V^\pi$ be the value function of any stationary Markovian policy $\pi$ on $\widebar{\mathcal M}$ and $\widebar{V}^\pi$ its value function based on the reward $\widebar{r}$. 
    We show that there exist a quantity $I(\widebar s)$ such that $\widebar{V}^\pi(\widebar s) = I(\widebar s) + V^\pi(\widebar s)$. 
    Since the quantity $I(\widebar s)$ does not depend on $\pi$, it means that the ordering of the policies in terms of value function and thus performance is preserved by using this new reward function. 
    Note that we can express the $t^{\text{th}}$ state in $\widebar{\mathcal M}$ as a tuple of element of the underlying MDP as 
    $\widebar s_t = (s_{t-\delay},a_{t-\delay},\dots,a_{t-1},s_{t})$.  We allow for negative indexing in the underlying MDP for the first term which are fixed by $\widebar{\mu}$.
    
    We now proceed to the write $\widebar{V}^\pi$ to introduce as a function of $V^\pi(\widebar s)$.
    \begin{align*}
        \widebar{V}^\pi  (\widebar s)
        &=  \E_{\substack{\widebar{s}_{t+1}\sim p(\cdot\vert \widebar{s}_t,a_t)\\a_t\sim\pi(\cdot\vert \widebar{s}_t)}}\left[\left. \sum_{t=0}^\infty\gamma^t \widebar{r}(\widebar{s}_t,a_t) \right\vert \widebar{s}_0=\widebar s\right]
        \\
        &= \E_{\substack{\widebar{s}_{t+1}\sim p(\cdot\vert \widebar{s}_t,a_t)\\a_t\sim\pi(\cdot\vert \widebar{s}_t)}}\left[\left. \sum_{t=0}^\infty\gamma^t r(s_{t-\delay},a_{t-\delay}) \right\vert \widebar{s}_0=\widebar s\right]
        \\
        &= \underbrace{\E_{\widebar{s}_{t+1}\sim p(\cdot\vert \widebar{s}_t,a_t)}\left[\left. \sum_{t=0}^{\delay-1}\gamma^t r(s_{t-\delay},a_{t-\delay}) \right\vert \widebar{s}_0=\widebar s\right]}_{I(\widebar s)}
        + \E_{\substack{\widebar{s}_{t+1}\sim p(\cdot\vert \widebar{s}_t,a_t)\\a_t\sim\pi(\cdot\vert \widebar{s}_t)}}\left[\left. \sum_{t=\delay}^\infty\gamma^t r(s_{t-\delay},a_{t-\delay}) \right\vert \widebar{s}_0=\widebar s\right].
    \end{align*}
    Two things have to be noted in the previous equation, first, the term on the left does not depend on $\pi$ anymore since the actions involved in the reward collection are already contained in $\widebar s$. It is our sought-after quantity $I$. Second, the reward in the term of the right can be interpreted are regular reward of a $\delay^{\text{th}}$-order MDP since $r(\widebar s_t,a_t)=r(s_t,a_t)$.
    
    Therefore,
    \begin{align*}
        \widebar{V}^\pi  (\widebar s)
        &= I(\widebar s) 
        + \E_{\substack{\widebar{s}_{t+1}\sim p(\cdot\vert \widebar{s}_t,a_t)\\a_t\sim\pi(\cdot\vert \widebar{s}_t)}}\left[\left. \sum_{t=\delay}^\infty\gamma^t r(\widebar{s}_{t},a_{t}) \right\vert \widebar{s}_0=\widebar s\right]
        \\
        &= I(\widebar s) 
        +  V^\pi(\widebar s).
    \end{align*}
    We found such a quantity $I$ to link $\widebar{V}^\pi$ and $V^\pi$, proving the statement.
\end{proof}

We are now able to prove a bound between delayed and undelayed value functions.
\begin{thm}
\label{th:order_d_bound}
    Consider an $(L_P,L_r)$-LC MDP $\mathcal{M}$ and its $\delay^{\text{th}}$-order MDP counterpart $\widebar{\mathcal{M}}$.  Let $\pi_E$ be a $L_\pi$-LC undelayed policy and assume that $\gamma L_P(1+L_\pi)\le 1$. 
    Let $\pidida$ be a delayed policy as defined in \Cref{eq:belief_pol}. 
    Then, for any $\widebar s\in\widebar{\Ss}$,
    \begin{align*}
        \left\vert V^{\pi_E} (\widebar s) - V^{\pidida} (\widebar s) \right\vert
        \leq 
        \frac{\gamma }{1-\gamma} L_\pi L_Q (1+L_P)   \sigma_b^{x},
    \end{align*}
    where $V^{\pi_E}$ and $V^{\pidida}$ are defined on $\widebar{\mathcal{M}}$ and $\sigma_b^{x} = \E_{\substack{x\sim d^{\pidida}_{x}\\ s,s'\sim b(\cdot\vert x)}}\left[ \dist_{\Ss}(s,s') \right]$ for $x$ the augmented state contained in $\widebar s$ and $d^{\pidida}_{x}$ being defined on the DMDP.
\end{thm}
\begin{proof}
    By \citep{puterman2014markov}, the two state value functions can be written as follows.
    \begin{align*}
        V^{\pi_E} (\widebar s) = \frac{1}{1-\gamma} \int_{\widebar{\Ss}} \int_{\widebar{\As}} r(\widebar s',a) \pi_E(a\vert \widebar s)d^{\pi_E}_{\widebar s}(\widebar s')~da~d\widebar s'.
    \end{align*}
    \begin{align*}
        V^{\pidida} (\widebar s) = \frac{1}{1-\gamma} \int_{\widebar{\Ss}} \int_{\widebar{\As}} r(\widebar s',a) \pidida(a\vert \widebar s)d^{\pidida}_{\widebar s'}(\widebar s)~da~d\widebar s'.
    \end{align*}
    
    Writing their difference gives
    \begin{align*}
        V^{\pi_E} (\widebar s) - V^{\pidida} (\widebar s) = \frac{1}{1-\gamma} \int_{\widebar{\Ss}}  \int_{\widebar{\As}} r(\widebar s',a)\left(\pi_E(a\vert \widebar s)d^{\pi_E}_{\widebar s}(\widebar s') - \pidida(a\vert \widebar s)d^{\pidida}_{\widebar s}(\widebar s')\right)~da~d\widebar s'.
    \end{align*}
    We now use \Cref{pp:order_d_reward} to remove the integral over the action space.
    \begin{align*}
        V^{\pi_E} (\widebar s) - V^{\pidida} (\widebar s) 
        &= \frac{1}{1-\gamma} \int_{\widebar{\Ss}}  \int_{\widebar{\As}} \widebar r(\widebar s')\left(\pi_E(a\vert \widebar s)d^{\pi_E}_{\widebar s'}(\widebar s) - \pidida(a\vert \widebar s)d^{\pidida}_{\widebar s}(\widebar s')\right)~da~d\widebar s'
        \\
        &= \frac{1}{1-\gamma} \int_{\widebar{\Ss}}  \widebar r(\widebar s')\left(d^{\pi_E}_{\widebar s}(\widebar s') - d^{\pidida}_{\widebar s}(\widebar s')\right)~d\widebar s'.
    \end{align*}
    One can now use the fact that $\widebar r$ is $L_r$-LC on $\widebar{\Ss}$ because $r$ is $L_r$-LC on $\Ss\times\As$. One thus gets
    \begin{align*}
        \left\vert V^{\pi_E} (\widebar s) - V^{\pidida} (\widebar s) \right\vert \leq \frac{L_r}{1-\gamma}\wass\left(d^{\pi_E}_{\widebar s}\Vert d^{\pidida}_{\widebar s}\right).
    \end{align*}
    
    Applying \Cref{th:state_distrib_bound} yields 
    \begin{align*}
        \left\vert V^{\pi_E} (\widebar s) - V^{\pidida} (\widebar s) \right\vert \leq \frac{\gamma }{1-\gamma}  L_\pi L_Q (1+L_P)
        \sigma_b^{x},
    \end{align*}
    where $x$ is the augmented state contained in $\widebar s$.
    This concludes the proof.
\end{proof}

As for \Cref{th:perf_diff_bound}, we can now use additional assumptions to bound $\sigma_b^{\mu}$.

\begin{cor}
\label{cor:order_d_bound_eucl}
    Under the conditions of \Cref{th:order_d_bound} and adding that $\Ss\subset\mathbb{R}^n$ is equipped with the Euclidean norm.
    Let $\pidida$ be a $\delay$-delayed policy as defined in \Cref{eq:belief_pol}. 
    Then, for any $\widebar s\in\widebar{\Ss}$,
    \begin{align*}
        \left\vert V^{\pi_E} (\widebar s) - V^{\pidida} (\widebar s) \right\vert
        \leq 
        \frac{\gamma }{1-\gamma} L_\pi L_Q (1+L_P)
        \E_{x'\sim d_x^{\tilde \pi}(\cdot)}\left[\sqrt{ \Var_{s\sim b(\cdot|x')}(s|x')}\right].
    \end{align*}
\end{cor}
\begin{proof}
    The result follows from application of \Cref{lem:bound_sigma_eucl} to \Cref{th:order_d_bound}.
\end{proof}

\begin{cor}
\label{cor:order_d_bound_tlc}
    Under the conditions of \Cref{th:order_d_bound} and adding that the MDP is $L_T$-TLC.
    Let $\pidida$ be a $\delay$-delayed policy as defined in \Cref{eq:belief_pol}. 
    Then,
    \begin{align*}
        \left\Vert V^{\pi_E} - V^{\pidida} \right\Vert_\infty  \leq  \frac{2\delay\gamma }{1-\gamma} L_T L_Q L_\pi  (1+L_P)
    \end{align*}
    where $V^{\pi_E}$ and $V^{\pidida}$ are defined on $\widebar{\Ss}$.
\end{cor}
\begin{proof}
    First, one applies \Cref{lem:bound_sigma_tlc} to \Cref{th:order_d_bound} to obtain
    \begin{align*}
        \left\vert V^{\pi_E} (\widebar s) - V^{\pidida} (\widebar s) \right\vert
        \leq \frac{2\delay\gamma }{1-\gamma} L_\pi L_Q (1+L_P),
    \end{align*}
    for some $\widebar s\in\widebar{\Ss}$. Then, taking the maximum over $\widebar{\Ss}$ gives the result since the \textit{rhs} doesn't depend on $\widebar s$.
\end{proof}

\subsection{Comparison of the Two Bounds}
\label{subsec:disc_bounds}
As stated in \Cref{sec:theoretical}, there are several choices for the quantities to consider when comparing undelayed to delayed performance. In this paper, \Cref{th:order_d_bound} provides a bound on the space of $\delay^{\text{th}}$-order MDP ($\widebar{\Ss}$) while \Cref{th:perf_diff_bound} compares a value function on the space of the augmented MDP ($\Xs$) to a value function on the classic state space ($\Ss$).

Recall the bounds for $\widebar s\in\widebar{\Ss}$ of \Cref{th:order_d_bound}
\begin{align}
    \left\vert V^{\pi_E} (\widebar s) - V^{\pidida} (\widebar s) \right\vert
    \leq 
    \frac{\gamma }{1-\gamma} L_\pi L_Q (1+L_P)   \sigma_b^{x},\label{eq:order_d_bound}
\end{align}
and for $x\in\mathcal{X}$ of \Cref{th:perf_diff_bound}
\begin{align}
    \E_{s\sim b(\cdot\vert x)}&[V^{\pi_E}(s)] - V^{\pidida}(x) \leq \frac{L_Q L_\pi}{1-\gamma} \sigma_b^x,\label{eq:perf_diff_bound}
\end{align}
where $\sigma_b^x = \E_{\substack{x'\sim d^{\pidida}_x\\ s,s'\sim b(\cdot\vert x')}}\left[ \dist_{\Ss}(s,s') \right]$.

As opposed to what the notations may suggest, the Lipschitz constant $L_Q$ doesn't have exactly the same value.  
In \Cref{eq:order_d_bound}, we recognised in the proof the $L_Q$ of the $Q$ function as given in \citet{rachelson2010locality} under the assumption that $\gamma L_P(1+L_\pi)\le 1$. In \Cref{eq:perf_diff_bound} however, we only assume that there exists such constant $L_Q$ for which the $Q$ function is $L_Q$-LC. That includes the case when $\gamma L_P(1+L_\pi)\le 1$ but is a more general result.

The other difference between the bounds lies in the factor $\gamma (L_P+1)$ of \Cref{eq:order_d_bound}. Depending on the task, this factor may be smaller or greater than 1, changing the order of the bounds.

\section{Experimental details}

\subsection{Imitation Loss and DIDA's Policy}
\label{app:dida_policy}
As stated in the main paper, the function learnt by DIDA can drift from \Cref{eq:belief_pol} depending on the class of policies of $\pi_I$ and the loss function that is used for the imitation step. We derive the policy learnt by DIDA for the two following cases.

\textbf{Mean squared error loss}~~DIDA is trained on
\begin{align*}
    \argmin_{\theta} \int_{\Ss}\int_{\As} (a-\pidel_\theta(x))^2 \pi_E(a\vert s) b(s\vert x) da~ds
\end{align*}
which is minimized for $\theta^*$ such that
\begin{align*}
    \pidel_{\theta^*}(x) = \int_{\Ss} \E_{a \sim\pi_E(\cdot \vert s)}[a]b(s\vert x)ds.
\end{align*}
That means that the policy learnt by DIDA outputs the mean value of the action given the belief and the expert policy distribution.

\textbf{Kullback-Leibler loss}~~DIDA is trained on
\begin{align}
    &\argmin_{\theta} \int_{\Ss} D_{KL}(\pi_E(\cdot \vert s) \Vert \pidel_\theta(\cdot \vert x)) b(s\vert x)ds\nonumber
    \\
    &= \argmin_{\theta} \int_{\Ss} \int_{\As} \pi_E(a \vert s) \log \pi_E(a \vert x)) b(s\vert x)da~ds - \int_{\Ss} \int_{\As} \pi_E(a \vert s) \log \pidel_\theta(a \vert x)) b(s\vert x)da~ds\nonumber
    \\
    &= \argmin_{\theta} - \int_{\Ss} \int_{\As} \pi_E(a \vert s) \log \pidel_\theta(a \vert x)) b(s\vert x)da~ds\label{eq:indpt_theta}
    \\
    &= \argmin_{\theta}  - \int_{\As}\int_{\Ss}  \pi_E(a \vert s) \log \pidel_\theta(a \vert x)) b(s\vert x)ds~da\label{eq:fubini}
    \\
    &= \argmin_{\theta}  \int_{\As} \int_{\Ss}  \pi_E(a \vert s) b(s\vert x)\log\left(b(s\vert x) \pi_E(a \vert x))\right) ds~da - \int_{\As}\int_{\Ss}  \pi_E(a \vert s) \log \pidel_\theta(a \vert x)) b(s\vert x)ds~da\nonumber
    \\
    &= \argmin_{\theta}  D_{KL}\left(\int_{\Ss}\pi_E(\cdot \vert s)b(s\vert x)ds , \pidel_\theta(\cdot \vert x)\right) ,\nonumber
\end{align}
where \Cref{eq:indpt_theta} holds because the first integral does not depend on $\theta$ and \Cref{eq:fubini} holds by Fubini's theorem since the functions inside the integral are always negative.

\subsection{Hyper-parameters}
\label{app:hyper_param}
For pendulum, the test performance are obtained from interacting 1000 steps with the environment, with maximum episode length of 200.
For {mujoco} environments, the number of steps is 1000 as well but the maximum episode length is 500.
The other of hyper-parameter are given for each approach, for each environment in the following tables.

\begin{table}[H]
\centering
\begin{tabular}{|l||l|l|l|}
\hline
Hyper-parameter & Pendulum & {Mujoco} & Trading \\ 
\hline
Policy type& Feed-forward&Feed-forward&Extra Trees
\\
Iterations&245&400&30
\\
Steps per iteration&10,000&10,000&2 years of data
\\
$\beta$ sequence&$\beta_1=1, \beta_{i\ge 2}=0$&$\beta_1=1, \beta_{i\ge 2}=0$&$\beta_1=1, \beta_{i\ge 2}=0$
\\
Max buffer size&10 iterations&10 iterations&Unlimited
\\
Policy neurons&$[100,100,10]$&$[100,100,100,10]$&$\emptyset$
\\
Activations&ReLU~\citep{nair2010rectified}&ReLU&$\emptyset$
\\
Optimizer&\vtop{\hbox{\strut Adam $(\beta_1=0.9,\beta_2=0.999)$}\hbox{\strut \citep{kingma2014adam}}}&Adam $(\beta_1=0.9,\beta_2=0.999)$&$\emptyset$
\\
Learning rate&$1e-3$&$1e-3$&$\emptyset$
\\
Batch size&$64$&$64$&$\emptyset$
\\
Min samples split&$\emptyset$&$\emptyset$&$100$
\\
n estimators&$\emptyset$&$\emptyset$&$200$
\\
\hline
\end{tabular}
\caption{Hyper-parameters for DIDA}
\end{table}


\begin{table}[H]
\centering
\begin{tabular}{|l||l|l|}
\hline
Hyper-parameter & Pendulum & {Mujoco} \\ 
\hline
Epochs&$2000$&$1000$\\
Steps per epoch&$5000$&$5000$\\
Pre-training epochs&$25$&$25$\\
Pre-training steps&$10000$&$10000$\\
Backtracking line search iterations&$10$&$10$\\
Backtracking line search step&$0.8$&$0.8$\\
Conjugate gradient iterations&$10$&$10$\\
Discount $\gamma$&$0.99$&$0.99$\\
Max KL divergence&$0.001$&$0.001$\\
$\lambda$ for GAE~\cite{schulman2015high}&$0.97$&$0.97$\\
Value function neurons&$[64]$&$[64]$\\
Value function iterations&$3$&$3$\\
Value function learning rate&$0.01$&$0.01$\\
Policy neurons&$[64,64]$&$[64,64]$\\
Activations&ReLU&ReLU\\
Encoder feed-forward neurons&$[8]$&$[8]$\\
Encoder learning rate&$0.01$&$0.01$\\
Encoder iterations&$2$&$2$\\
Encoder dimension&$64$&$64$\\
Encoder heads&$2$&$2$\\
Encoder optimizer&Adam $(\beta_1=0.9,\beta_2=0.999)$&Adam $(\beta_1=0.9,\beta_2=0.999)$\\
Layers of MAF~\citep{papamakarios2017masked}&$5$&$5$\\
Maf neurons&$[16]$&$[16]$\\
MAF learning rate&$0.01$&$0.01$\\
MAF optimizer&Adam $(\beta_1=0.9,\beta_2=0.999)$&Adam $(\beta_1=0.9,\beta_2=0.999)$\\
Epochs of training the belief&$200$&$200$\\
Batch size belief learning&$10000$&$1000$\\
Prediction buffer size&$100000$&$100000$\\
Belief representation dimension&$8$&$32$\\
\hline
\end{tabular}
\caption{Hyper-parameters for D-TRPO}
\end{table}

\begin{table}[H]
\centering
\begin{tabular}{|l||l|}
\hline
Hyper-parameter & Pendulum \\ 
\hline
Discount $\gamma$&$0.99$\\
Initial replay size&$64$\\
Buffer size&$50000$\\
Batch size&$64$\\
Actor $\mu$ neurons&$256$\\
Actor $\sigma$ neurons&$256$\\
Actor optimizer&Adam $(\beta_1=0.9,\beta_2=0.999)$\\
Warmup transitions&$100$\\
Polyak update $\tau$&$0.005$\\
Entropy learning rate&$3e-4$\\
Train frequency&$50$\\
\hline
\end{tabular}
\caption{Hyper-parameters for M-SAC and A-SAC}
\end{table}

\begin{table}[H]
\centering
\begin{tabular}{|l||l|l|}
\hline
Hyper-parameter & Pendulum & {Mujoco} \\ 
\hline
Epochs&$2000$&$1000$\\
Steps per epoch&$5000$&$5000$\\
Pre-training epochs&$2$&$2$\\
Pre-training steps&$10000$&$10000$\\
Backtracking line search iterations&$10$&$10$\\
Backtracking line search step&$0.8$&$0.8$\\
Conjugate gradient iterations&$10$&$10$\\
Discount $\gamma$&$0.99$&$0.99$\\
Max KL divergence&$0.001$&$0.001$\\
$\lambda$ for GAE&$0.97$&$0.97$\\
Value function neurons&$[64]$&$[64]$\\
Value function iterations&$3$&$3$\\
Value function learning rate&$0.01$&$0.01$\\
Policy neurons&$[64,64]$&$[64,64]$\\
Activations&ReLU&ReLU\\
Encoder feed-forward neurons&$[8]$&$[8]$\\
Encoder learning rate&$5e-3$&$5e-3$\\
Encoder iterations&$1$&$1$\\
Encoder dimension&$64$&$64$\\
Encoder heads&$2$&$2$\\
Encoder optimizer&Adam $(\beta_1=0.9,\beta_2=0.999)$&Adam $(\beta_1=0.9,\beta_2=0.999)$\\
Batch size belief learning&$10000$&$1000$\\
Prediction buffer size&$100000$&$100000$\\
Belief representation dimension&$8$&$32$\\
\hline
\end{tabular}
\caption{Hyper-parameters for L2-TRPO}
\end{table}

\begin{table}[H]
\centering
\begin{tabular}{|l||l|l|}
\hline
Hyper-parameter & Pendulum & {Mujoco} \\ 
\hline
Epochs&$2000$&$1000$\\
Steps per epoch&$5000$&$5000$\\
Pre-training epochs&$2$&$2$\\
Pre-training steps&$10000$&$10000$\\
Backtracking line search iterations&$10$&$10$\\
Backtracking line search step&$0.8$&$0.8$\\
Conjugate gradient iterations&$10$&$10$\\
Discount $\gamma$&$0.99$&$0.99$\\
Max KL divergence&$0.001$&$0.001$\\
$\lambda$ for GAE &$0.97$&$0.97$\\
Value function neurons&$[64]$&$[64]$\\
Value function iterations&$3$&$3$\\
Value function learning rate&$0.01$&$0.01$\\
Policy neurons&$[64,64]$&$[64,64]$\\
Activations&ReLU&ReLU\\
\hline
\end{tabular}
\caption{Hyper-parameters for M-TRPO and A-TRPO}
\end{table}

\begin{table}[H]
\centering
\begin{tabular}{|l||l|}
\hline
Hyper-parameter & Pendulum \\ 
\hline
Epochs&$2000$\\
Steps per epoch&$5000$\\
Discount $\gamma$&$0.99$\\
Eligibility trace $\lambda$&$0.9$\\
Learning rate&$0.1$\\
$\epsilon$-greedy parameter&$0.2$\\
$\Ss$ discretization size&$15$\\
$\As$ discretization size&$3$\\
\hline
\end{tabular}
\caption{Hyper-parameters for SARSA}
\end{table}

\begin{table}[H]
\centering
\begin{tabular}{|l||l|}
\hline
Hyper-parameter & Pendulum \\ 
\hline
Epochs&$2000$\\
Steps per epoch&$5000$\\
Discount $\gamma$&$0.99$\\
Eligibility trace $\lambda$&$0.9$\\
Learning rate&$0.1$\\
$\epsilon$-greedy parameter&$0.2$\\
$\Ss$ discretization size&$15$\\
$\As$ discretization size&$3$\\
\hline
\end{tabular}
\caption{Hyper-parameters for dSARSA}
\end{table}

\subsection{Further experiments}
\label{subsec:further_exp}

\paragraph{Stochastic Pendulum}~~In this experiment, we evaluate DIDA on a stochastic environment. We follow \citet{liotet2021learning} and add stochasticity to the pendulum environment. 
In order to do so, to the action selected by the agent, we add an i.i.d. noise of the form $\epsilon = \text{scale}( \eta+ \text{shift})$
where $\eta$ is some probability distribution. We construct 6 such noises reported in \Cref{tab:noises}. For readability of the plots, we group noises by similarity.
We build an additional noise, referred to as \emph{uniform} noise, which follows the action of the agent with probability 0.9 and otherwise samples an action uniformly at random inside the action space. We place this noise in group 2. The results are obtained with the hyper-parameters given in \Cref{app:hyper_param} for the pendulum environment.

\begin{table}[h]
\centering
\begin{tabular}{|l||l|l|l|l|}
\hline
Noise & Distribution  $\eta$ & Shift & Scale & Group \\ \hline
\textit{Beta (8,2)}& $\beta (8,2)$ & $0.5$ & $2$ &  1\\ 
\textit{Beta (2,2)} & $\beta (2,2)$ & $0.5$ & $2$ & 1\\ 
\textit{U-Shaped} & $\beta (0.5,0.5)$ & $0.5$ & $2$ & 1\\ 
\textit{Triangular} & $\text{Triangular} (-2,1,2)$ & $0$ & $1$ & 2\\ 
\textit{Lognormal (1)} & $\text{Lognormal} (0,1)$ & $-1$ & $1$ & 3\\ 
\textit{Lognormal (0.1)} & $\text{Lognormal} (0,0.1)$ & $-1$ & $1$ & 3\\ 
\hline
\end{tabular}
\caption{Distributions for the noise added to the action in the stochastic Pendulum.}
\label{tab:noises}
\end{table}

\begin{figure*}[t]
\centering
    \begin{subfigure}{0.33\textwidth}
        \centering
        \includegraphics[width=\linewidth]{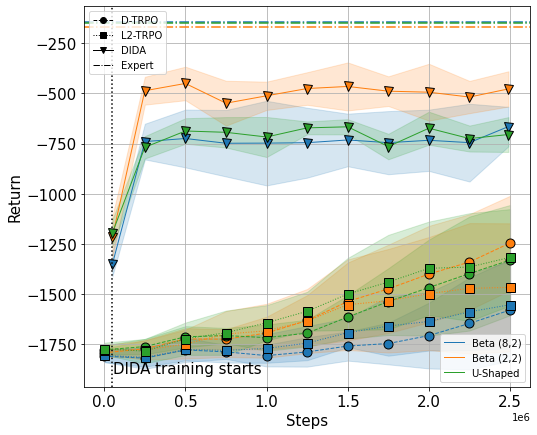}
        \caption{Noises of group 1.}
        \label{fig:group_1}
    \end{subfigure}%
\hfill
    \begin{subfigure}{0.33\textwidth}
        \centering
        \includegraphics[width=\linewidth]{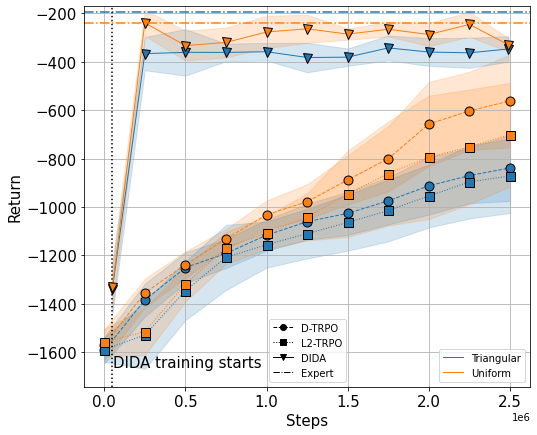}
        \caption{Noises of group 2.}
        \label{fig:group_2}
    \end{subfigure}
\hfill
    \begin{subfigure}{0.33\textwidth}
        \centering
        \includegraphics[width=\linewidth]{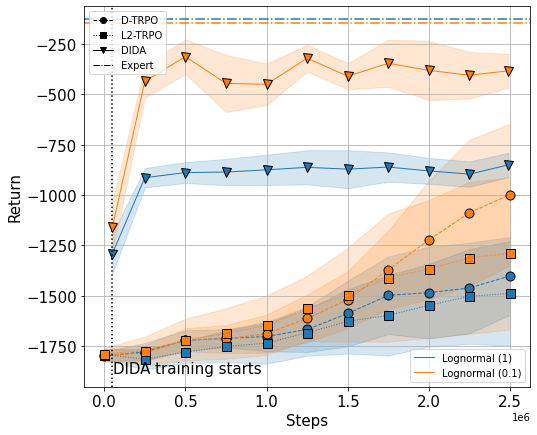}
        \caption{Noises of group 3.}
        \label{fig:group_3}
    \end{subfigure}
\caption{Mean return and one standard deviation (shaded) as a function of the number of steps sampled from the environment for stochastic Pendulum with different noises (10 seeds).}
\end{figure*}

\begin{itemize}
    \item \textbf{Group 1}: In all these cases, we are considering noises based on beta distributions, with different parameters. We can see in \Cref{fig:group_1} that our algorithm is able to achieve a much better performance compared to the ones of the baselines, even with a fraction of the training samples. For every algorithm, the most favourable case seems to be the second one, based on a beta $(2,2)$ noise. This may be because of the features of the other two noises. The first, beta $(8,2)$, is non-zero mean, so that the action is affected, on average, by a translation in one direction. The third one, based on a $\beta(0.5,0.5)$ distribution, is zero-mean, but is characterized by a higher variance than the second ($0.5$ vs $0.2$).
    \item \textbf{Group 2}:
    Here again, we see in \Cref{fig:group_2} that DIDA is able to get the best performance, even if the two baselines seem to learn much faster than for group 1. 
    This suggest that group 2 contains easier tasks, even though the triangular noise is not symmetric. Still, note that even if the probability of the random action in the first case is small, in the situation of delay it accumulates so that the probability of having a random action inside the action queue of $d=5$ is $1-0.9^5\approx 0.41$. Nonetheless, DIDA seems to deal very well with this situation.
    \item \textbf{Group 3}:
    Being strongly asymmetric and unbounded, theses noises pose more challenge to the algorithms. We report the results in \Cref{fig:group_3}.
    For the Lognormal (1) noise, no algorithm reaches a satisfactory performance.
    However, again, DIDA obtains the best performance for each noise.
\end{itemize}

\paragraph{Complement for trading}~~We illustrate here a specific problem to the trading task, caused by the batch-RL scenario. After some iterations, the policy trained by DIDA overfits the policy of the expert on the training set and its policy on the testing set starts to shift away from the one of the expert. We illustrate this by comparing the plots of the policies for the training set of years 2016-2017 (\Cref{fig:train_policies}) to the one of the testing set of year 2019 (\Cref{fig:test_policies}). In these plots, the x-axis represents the time of the day while the y-axis represents the day of the year. The color refers to the action of the agent. This representation clearly shows the daily pattern that the expert has found in the data.

\begin{figure*}[t]
    \centering
    \begin{subfigure}[b]{0.475\textwidth}
        \centering
        \includegraphics[width=\textwidth]{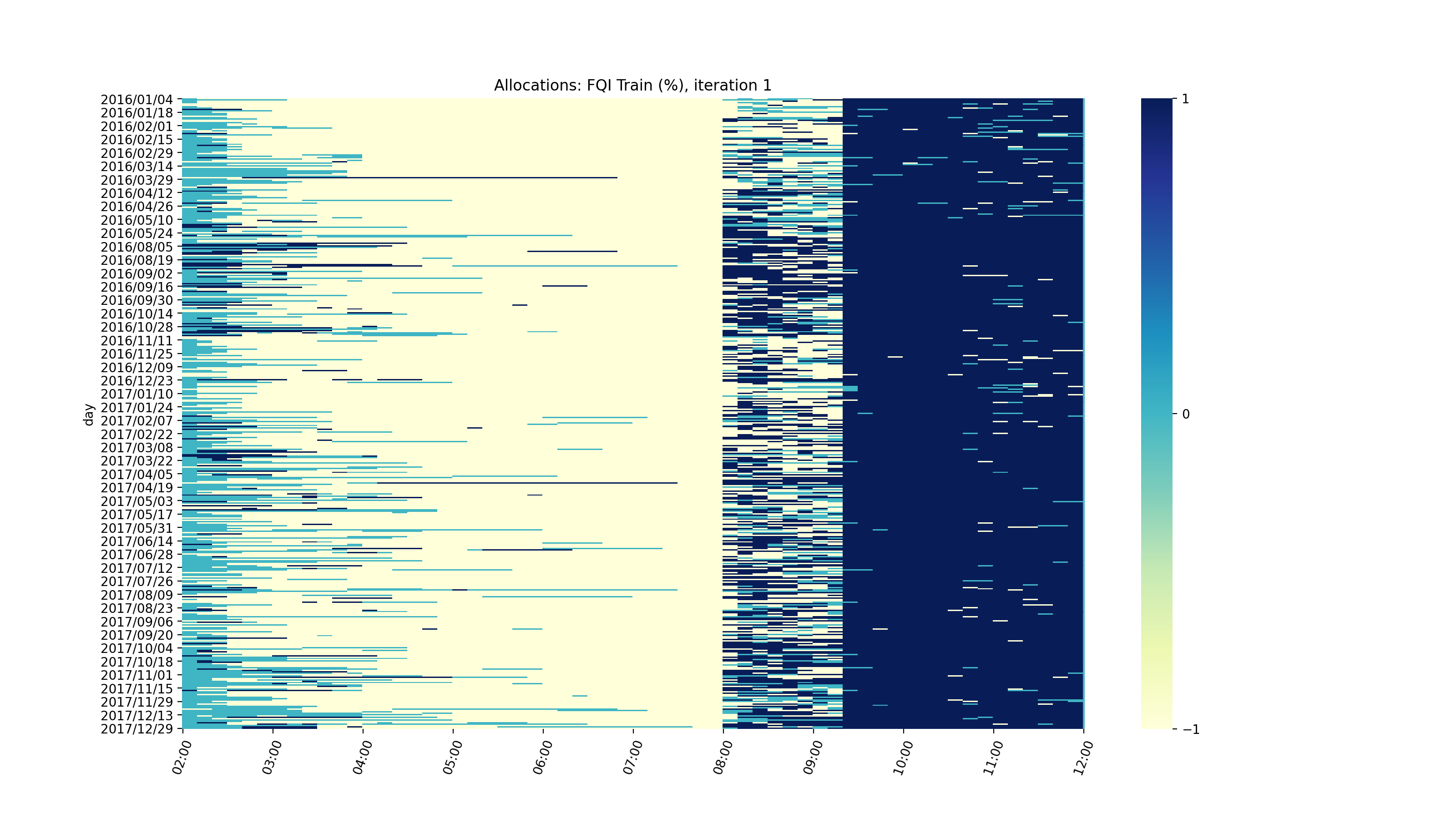}
        \caption{Expert}%
        \label{fig:expert_train}
    \end{subfigure}
    \hfill
    \begin{subfigure}[b]{0.475\textwidth}  
        \centering 
        \includegraphics[width=\textwidth]{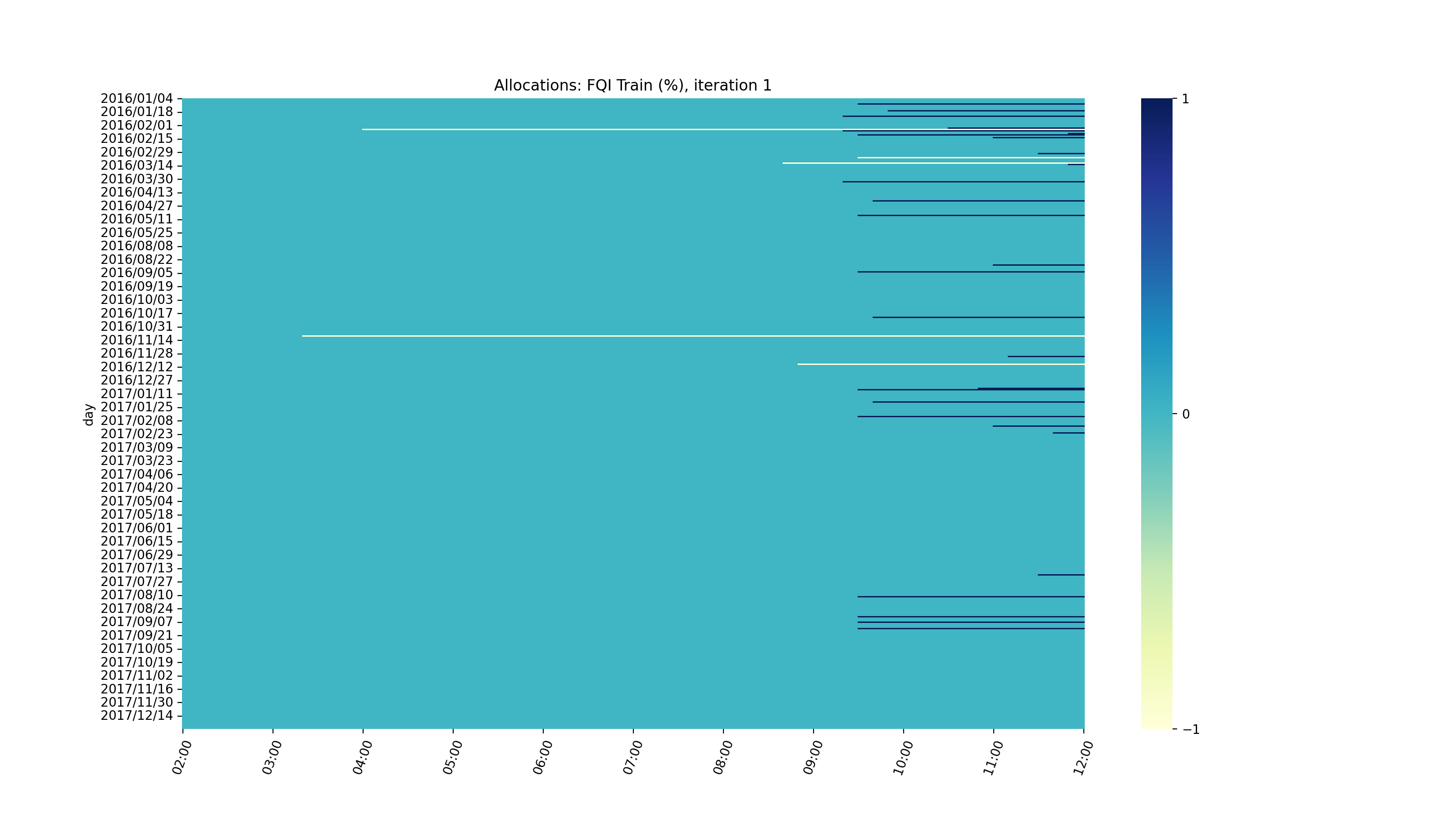}
        \caption{DIDA - Iteration 1}
        \label{fig:dida_train_1}
    \end{subfigure}
    \vskip\baselineskip
    \begin{subfigure}[b]{0.475\textwidth}   
        \centering 
        \includegraphics[width=\textwidth]{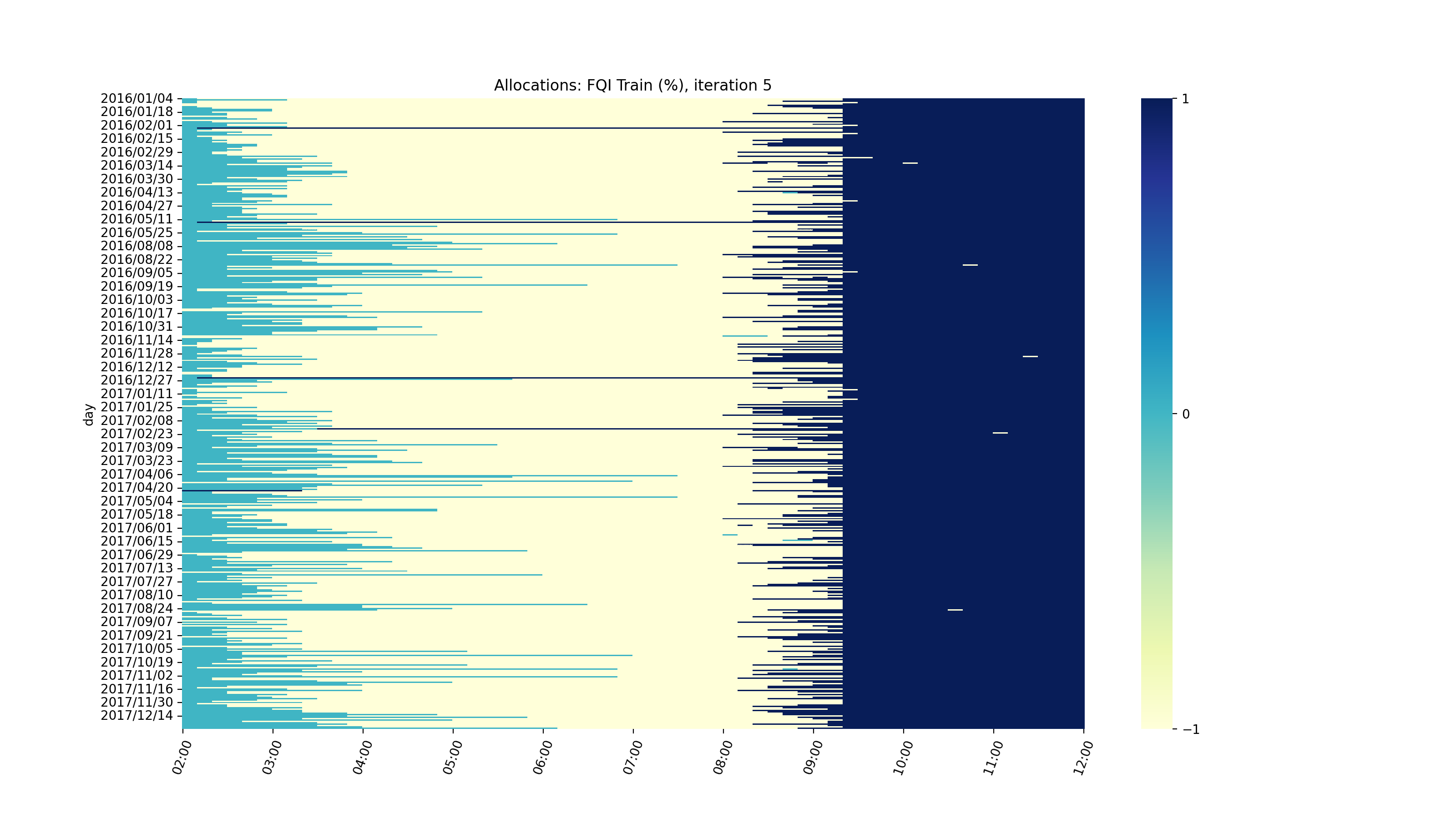}
        \caption{DIDA - Iteration 5}%
        \label{fig:dida_train_15}
    \end{subfigure}
    \hfill
    \begin{subfigure}[b]{0.475\textwidth}   
        \centering 
        \includegraphics[width=\textwidth]{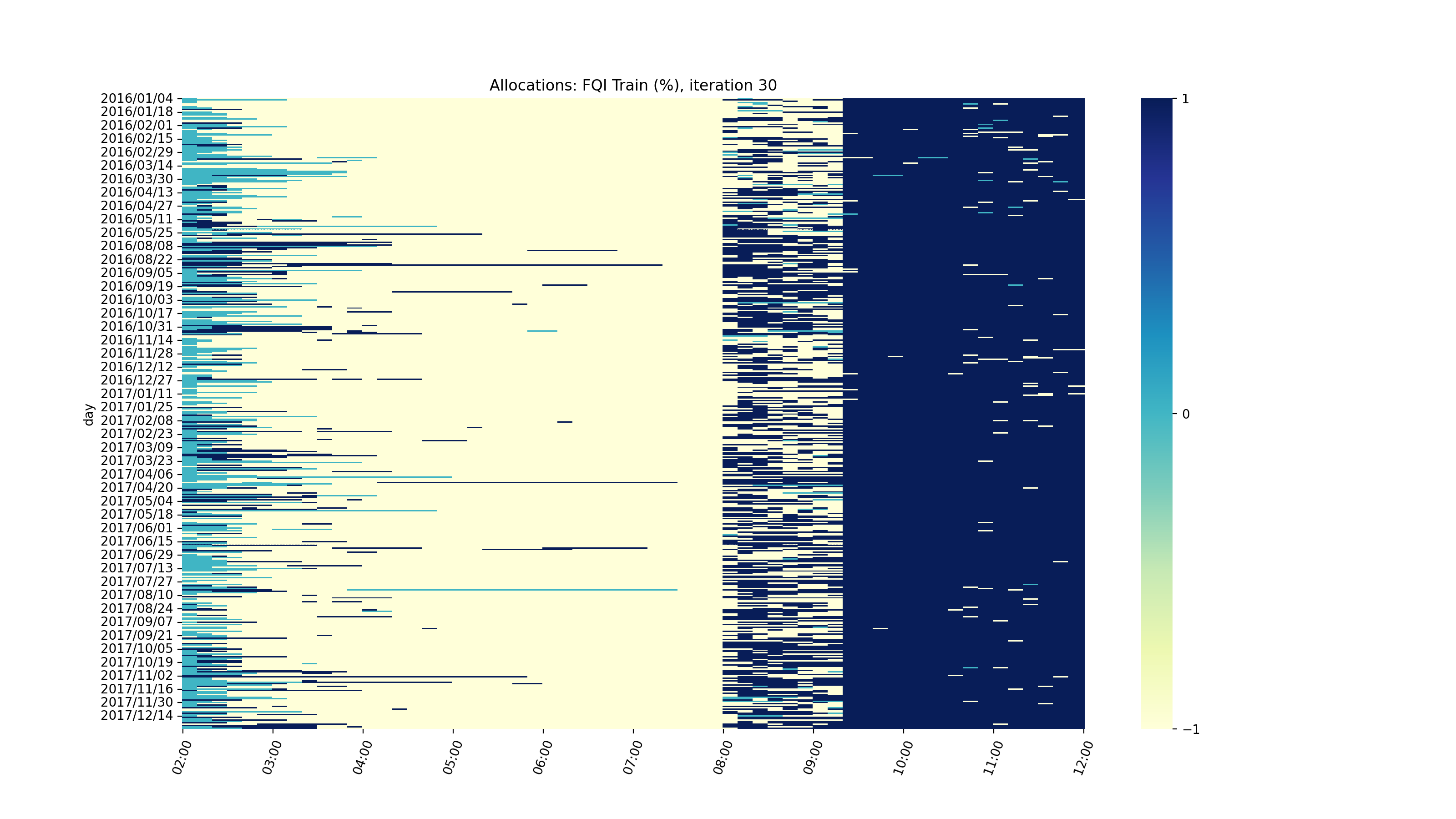}
        \caption{DIDA - Iteration 30}%
        \label{fig:dida_train_30}
    \end{subfigure}
\caption{Comparison of the expert's and DIDA's policies on the training set (2017-2016)}
\label{fig:train_policies}
\end{figure*}

\begin{figure*}[t]
    \centering
    \begin{subfigure}[b]{0.475\textwidth}
        \centering
        \includegraphics[width=\textwidth]{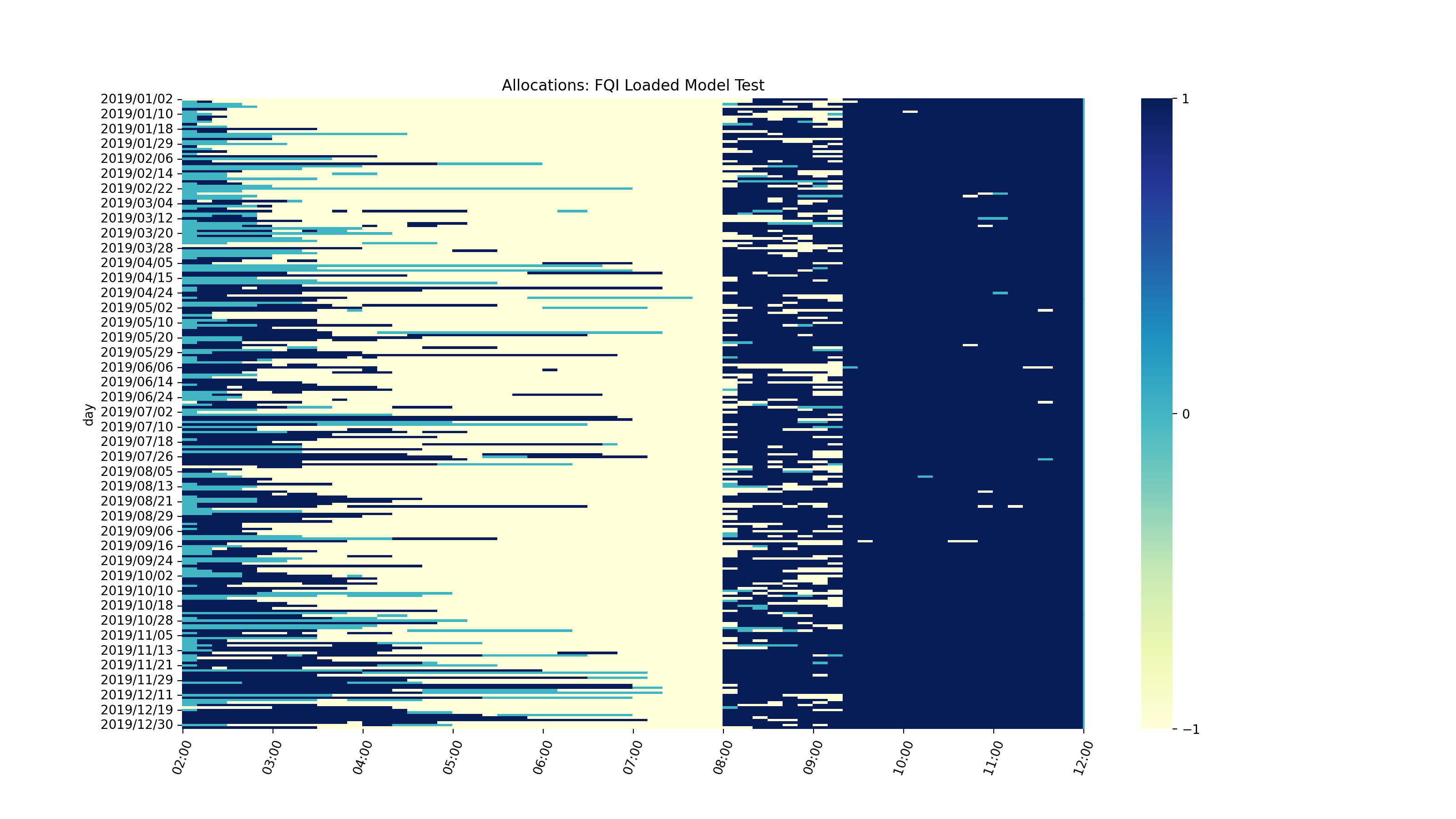}
        \caption{Expert}%
        \label{fig:expert_test}
    \end{subfigure}
    \hfill
    \begin{subfigure}[b]{0.475\textwidth}  
        \centering 
        \includegraphics[width=\textwidth]{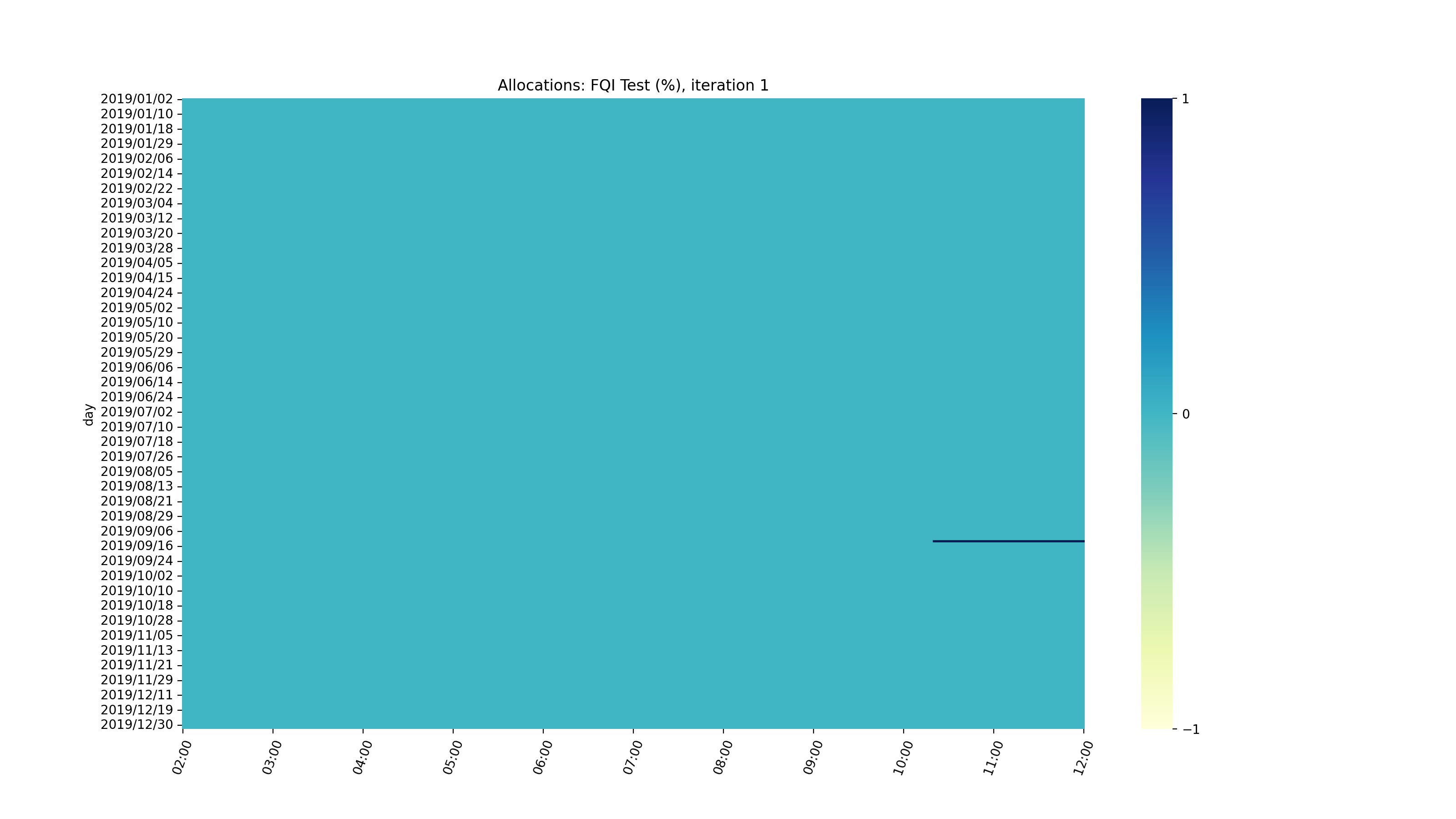}
        \caption{DIDA - Iteration 1}
        \label{fig:dida_test_1}
    \end{subfigure}
    \vskip\baselineskip
    \begin{subfigure}[b]{0.475\textwidth}   
        \centering 
        \includegraphics[width=\textwidth]{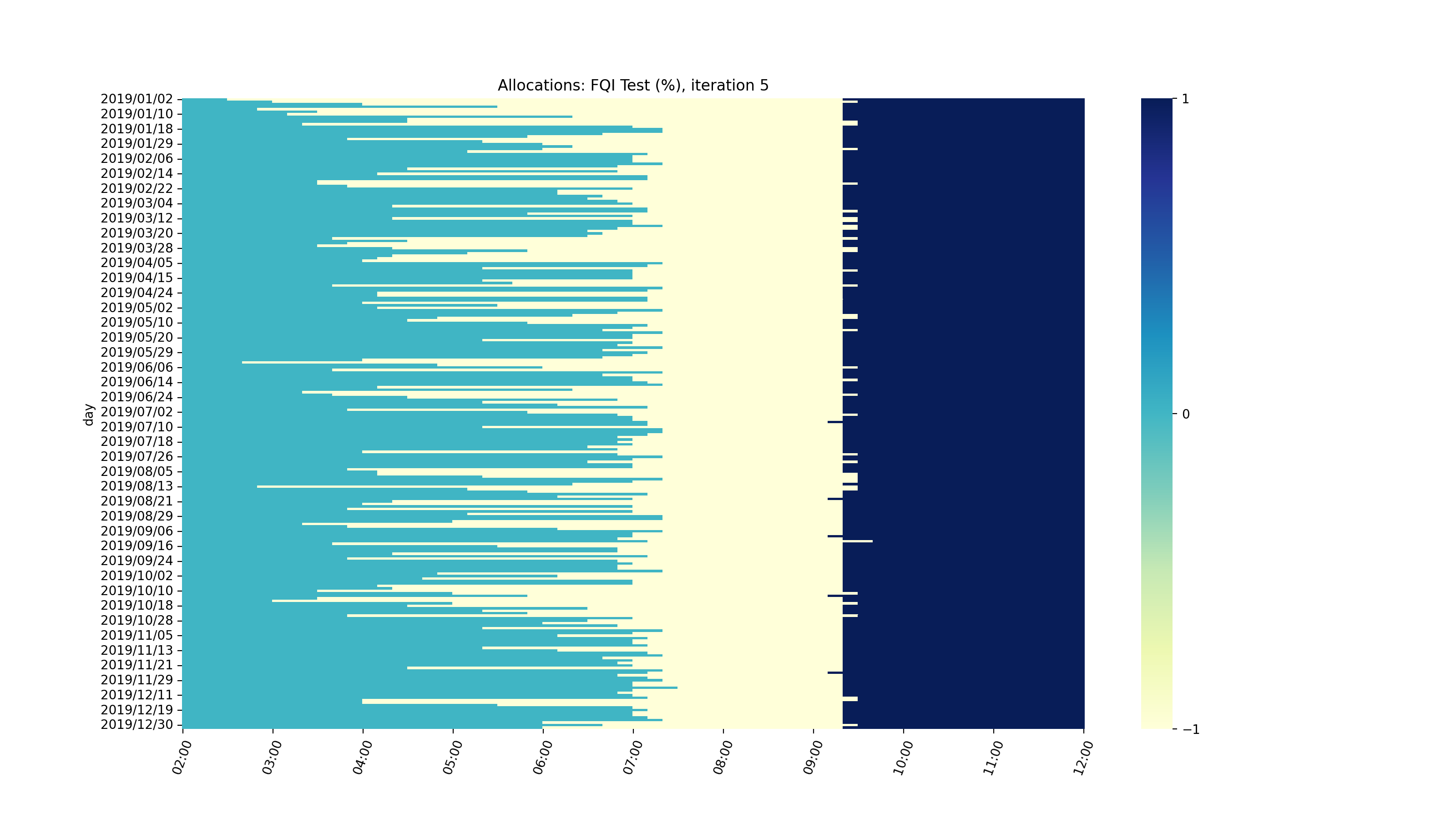}
        \caption{DIDA - Iteration 5}%
        \label{fig:dida_test_15}
    \end{subfigure}
    \hfill
    \begin{subfigure}[b]{0.475\textwidth}   
        \centering 
        \includegraphics[width=\textwidth]{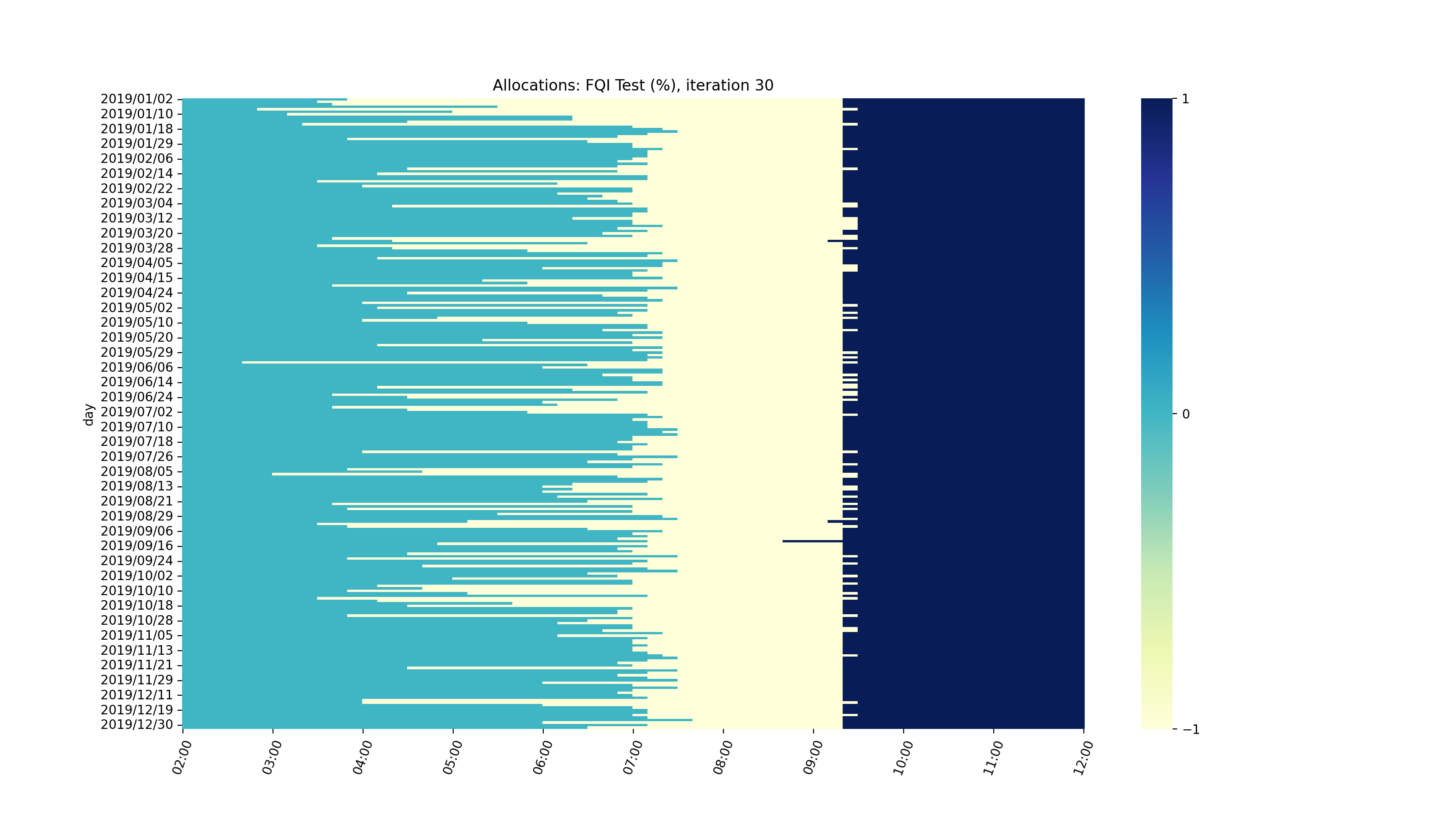}
        \caption{DIDA - Iteration 30}%
        \label{fig:dida_test_30}
    \end{subfigure}
\caption{Comparison of the expert's and DIDA's policies on the testing set (2019)}
\label{fig:test_policies}
\end{figure*}

\end{document}